%% file: arxiv-faulty.tex
\documentclass[a4paper,11pt]{amsart}
\usepackage{lmodern}  

\usepackage[T1]{fontenc}    % use 8-bit T1 fonts
\usepackage[top=1in, bottom=1in, left=1in, right=1in]{geometry}
\usepackage{hyperref}       % hyperlinks
\usepackage{url}            % simple URL typesetting
\usepackage{booktabs}       % professional-quality tables
\usepackage{amsmath,amssymb,amsfonts,amsthm,graphicx}       % blackboard math symbols
\usepackage{nicefrac}       % compact symbols for 1/2, etc.
\usepackage{microtype}      % microtypography

\usepackage{verbatim}
\usepackage{float}
\usepackage{enumerate}
\usepackage{enumitem}

\usepackage{algorithm}
\usepackage{algpseudocode}
\usepackage{multirow, array}

\title{Clustering with Noisy Queries}

\author{Arya Mazumdar}
\author{Barna Saha}\thanks{The authors are with the College of Information and Computer Sciences, University of Massachusetts Amherst, emails: \url{arya@cs.umass.edu},\url{barna@cs.umass.edu}.This work is partially supported by NSF awards CCF 1642658,  CCF 1642550,  CCF 1464310, a Yahoo ACE Award and a Google Faculty Research Award.}

\makeatletter
\renewcommand\subsubsection{\@startsection{subsubsection}{2}%
  \z@{.5\linespacing\@plus.7\linespacing}{-.5em}%
  {\normalfont\bfseries}}
\makeatother
\let\svthefootnote\thefootnote
\usepackage{color}

\usepackage{hyperref}
\usepackage{caption}
\usepackage{microtype}

\DeclareGraphicsExtensions{.eps}

\newcommand\nc\newcommand
\nc\bfa{{\boldsymbol a}}\nc\bfA{{\boldsymbol A}}\nc\cA{{\mathcal A}}
\nc\bfb{{\boldsymbol b}}\nc\bfB{{\boldsymbol B}}\nc\cB{{\mathcal B}}
\nc\bfc{{\boldsymbol c}}\nc\bfC{{\boldsymbol C}}\nc\cC{{\mathcal C}}
\nc\sC{{\mathscr C}}
\nc\bfd{{\boldsymbol d}}\nc\bfD{{\boldsymbol D}}\nc\cD{{\mathcal D}}
\nc\bfe{{\boldsymbol e}}\nc\bfE{{\boldsymbol E}}\nc\cE{{\mathcal E}}
\nc\bff{{\boldsymbol f}}\nc\bfF{{\boldsymbol F}}\nc\cF{{\mathcal F}}
\nc\bfg{{\boldsymbol g}}\nc\bfG{{\boldsymbol G}}\nc\cG{{\mathcal G}}
\nc\bfh{{\boldsymbol h}}\nc\bfH{{\boldsymbol H}}\nc\cH{{\mathcal H}}
\nc\bfi{{\boldsymbol i}}\nc\bfI{{\boldsymbol I}}\nc\cI{{\mathcal I}}
\nc\bfj{{\boldsymbol j}}\nc\bfJ{{\boldsymbol J}}\nc\cJ{{\mathcal J}}
\nc\bfk{{\boldsymbol k}}\nc\bfK{{\boldsymbol K}}\nc\cK{{\mathcal K}}
\nc\bfl{{\boldsymbol l}}\nc\bfL{{\boldsymbol L}}\nc\cL{{\mathcal L}}
\nc\bfm{{\boldsymbol m}}\nc\bfM{{\boldsymbol M}}\nc\sM{{\mathscr M}}\nc\cM{{\mathcal M}}
\nc\bfn{{\boldsymbol n}}\nc\bfN{{\boldsymbol N}}\nc\cN{{\mathcal N}}
\nc\bfo{{\boldsymbol o}}\nc\bfO{{\boldsymbol O}}\nc\cO{{\mathcal O}}
\nc\bfp{{\boldsymbol p}}\nc\bfP{{\boldsymbol P}}\nc\cP{{\mathcal P}}
\nc\bfq{{\boldsymbol q}}\nc\bfQ{{\boldsymbol Q}}\nc\cQ{{\mathcal Q}}
\nc\bfr{{\boldsymbol r}}\nc\bfR{{\boldsymbol R}}\nc\cR{{\mathcal R}}
\nc\bfs{{\boldsymbol s}}\nc\bfS{{\boldsymbol S}}\nc\cS{{\mathcal S}}
\nc\bft{{\boldsymbol t}}\nc\bfT{{\boldsymbol T}}\nc\cT{{\mathcal T}}
\nc\bfu{{\boldsymbol u}}\nc\bfU{{\boldsymbol U}}\nc\cU{{\mathcal U}}
\nc\bfv{{\boldsymbol v}}\nc\bfV{{\boldsymbol V}}\nc\cV{{\mathcal V}}
\nc\bfw{{\boldsymbol w}}\nc\bfW{{\boldsymbol W}}\nc\cW{{\mathcal W}}
\nc\bfx{{\boldsymbol x}}\nc\bfX{{\boldsymbol X}}\nc\cX{{\mathcal X}}
\nc\bfy{{\boldsymbol y}}\nc\bfY{{\boldsymbol Y}}\nc\cY{{\mathcal Y}}
\nc\bfz{{\boldsymbol z}}\nc\bfZ{{\boldsymbol Z}}\nc\cZ{{\mathcal Z}}

\nc\diff{{\mathrm d}}
\nc\e{{\mathrm e}}
\nc\calC{{\mathcal C}}
\newcommand{\remove}[1]{}

\newcommand{\avg}{{\mathbb E}}

\newtheorem*{lemma*}{Lemma}

\newtheorem{theorem}{Theorem}
\newtheorem*{theorem*}{Theorem}
\newtheorem{lemma}{Lemma}
\theoremstyle{definition}

\theoremstyle{corollaryn}
\newtheorem*{corollaryn}{Corollary}

\newtheorem*{problem}{Problem}

\newtheorem{remark}{Remark}
\newtheorem{claim}{Claim}

\newcommand{\cc}{{\sf Query-Cluster}}

\def\DEBUG{true}

\ifdefined\DEBUG

  \def\rem#1{{\marginpar{\raggedright\scriptsize #1}}}
  \newcommand{\barnr}[1]{\rem{\textcolor{red}{$\bullet$ #1}}}
  \newcommand{\aryar}[1]{\rem{\textcolor{green}{$\bullet$ #1}}}
\else

  \newcommand{\barnr}[1]{}
  \newcommand{\aryar}[1]{}

\fi

\newcommand\reals{{\mathbb R}}

\allowdisplaybreaks

\begin{document} 

\maketitle

\begin{abstract} 
In this paper, we initiate a rigorous theoretical study of clustering with noisy queries (or a faulty oracle). Given a set of $n$ elements, our goal is to recover the true clustering by asking minimum number of pairwise queries to an oracle. Oracle can answer queries of the form ``do elements $u$ and $v$ belong to the same cluster?''-the queries can be asked interactively (adaptive queries), or non-adaptively up-front, but its answer can be erroneous with probability $p$. In this paper, we provide the first information theoretic lower bound on the number of queries for clustering with noisy oracle in both situations.  We design novel algorithms that closely match this query complexity lower bound, even when the number of clusters is unknown.  Moreover, we design computationally efficient algorithms both for the adaptive and non-adaptive settings.
The problem captures/generalizes multiple application scenarios. It is directly motivated by the growing body of work that use crowdsourcing for {\em entity resolution}, a fundamental and challenging data mining task aimed to identify all records in a database referring to the same entity. Here crowd represents the noisy oracle, and the number of queries directly relates to the cost of crowdsourcing. Another application comes from the problem of {\em sign edge prediction} in social network, where social interactions can be both positive and negative, and one must identify the sign of all pair-wise interactions by  querying a few pairs. Furthermore, clustering with noisy oracle is intimately connected to correlation clustering, leading to improvement therein. Finally, it introduces a new direction of study in the popular {\em stochastic block model} where one has an incomplete stochastic block model matrix to recover the clusters.
\end{abstract} 
\input{intro-v2}

\input{faulty-oracle-v5}
\input{algo}

% \newpage 
\bibliographystyle{abbrv}

{%\small
\bibliography{bibfile}
}

%\newpage

\appendix
%{\Large \bf \begin{center} Supplementary Material \end{center}}
%\input{lemma-lb-v3}
%\input{appendix-faulty}
\input{appendix-algo}
\end{document}

%% file: intro-v2.tex
\section{Introduction}
%\subsection{Faulty oracle} %and round complexity}

Clustering is one of the most fundamental and popular methods for data classification. 
In this paper we initiate a rigorous theoretical study of clustering with the help of a `faulty oracle', a model that captures many application scenarios and has drawn significant attention in recent years\let\thefootnote\relax\footnotetext{A prior version of some of the results of this work appeared in arxiv previously~\cite[Sec.~6]{mazumdar2016clustering}, see \url{https://arxiv.org/abs/1604.01839}. In this version we rewrote several proofs for clarity, and included many new results.}.

\let\thefootnote\svthefootnote
Suppose we are given a set of $n$ points, that need to be clustered into $k$ clusters where $k$ is unknown to us. Suppose there is an oracle that can answer pair-wise queries of the form, ``do $u$ and $v$ belong to the same cluster?''. Repeating the same question to the oracle always returns the same answer, but the answer could be wrong with probability $p < \frac{1}{2}$ (that is slightly better than random answer).  We are interested to find the minimum number of queries needed to recover the true clusters with high probability.  Understanding query complexity of the faulty oracle model is a fundamental theoretical question \cite{frpu:94} with many existing works on sorting and selection \cite{bm:08, bm:09} where queries are erroneous with probability $p$, and repeating the same question does not change the answer. Here we study the basic clustering problem under this setting which also captures several fundamental applications. Throughout the paper, `noisy oracle' and `faulty oracle' have the same meaning. 

\vspace{0.1in}
\noindent{\bf Crowdsourced Entity Resolution.} 
Entity resolution (ER) is an important data mining task that tries to identify all records in a database that refer to the same underlying entity. Starting with the seminal work of Fellegi and Sunter \cite{fellegi1969theory}, numerous algorithms with variety of techniques have been developed for ER  \cite{elmagarmid2007duplicate,getoor2012entity,larsen2001iterative,christen2012data}. 
Still, due to ambiguity in representation and poor data quality, accuracy of automated ER techniques has been unsatisfactory. To remedy this, a recent trend in ER has been to use human in the loop. In this setting, humans are asked simple pair-wise queries adaptively, ``do $u$ and $v$ represent the same entity?'', and these answers are used to improve the final accuracy \cite{gokhale2014corleone,vesdapunt2014crowdsourcing,wang2012crowder,fss:16,DBLP:conf/icde/VerroiosG15,dalvi2013aggregating,ghosh2011moderates,karger2011iterative,aaai:17}. 
Proliferation of crowdsourcing platforms like Amazon Mechanical Turk (AMT), CrowdFlower etc. allows for easy implementation. However, data collected from non-expert workers on crowdsourcing platforms are inevitably noisy. A simple scheme to reduce errors could be to take a majority vote after asking the same question to multiple independent crowd workers. However, often that is not sufficient. Our experiments on several real datasets %(see experimentation section for details) 
with answers collected from AMT \cite{DBLP:journals/corr/GruenheidNKGK15,DBLP:conf/icde/VerroiosG15} show majority voting could sometime even increase the errors.
Interestingly, such an observation has been made by a recent paper as well \cite{nature}. There are more complex querying model \cite{nature,vinayak2016crowdsourced,DBLP:conf/icde/VerroiosG17}, and involved heuristics \cite{DBLP:journals/corr/GruenheidNKGK15,DBLP:conf/icde/VerroiosG15} to handle errors in this scenario. Let $p, 0 < p < 1/2$ be the probability of error\footnote{an approximation of $p$ can often be estimated manually from a small sample of crowd answers.} of a query answer which might also be the aggregated answer after repeating the query several times. Therefore, once the answer has been aggregated, it cannot change. In all crowdsourcing works, the goal is to minimize the number of queries to reduce the cost and time of crowdsourcing, and recover the entities (clusters). This is exactly clustering with noisy oracle. While several heuristics have been developed \cite{DBLP:conf/icde/VerroiosG15,gokhale2014corleone,DBLP:conf/icde/VerroiosG17}, here we provide a rigorous theory with near-optimal algorithms and hardness bounds.

%To overcome this, \cite{DBLP:journals/corr/GruenheidNKGK15} suggests repeatedly asking the same question and aggregating the answer via voting.  Still, as observed in \cite{DBLP:conf/icde/VerroiosG15,DBLP:conf/icde/VerroiosG17}, repetition reduces errors only by $\sim 20\%$, and even the aggregated answer can be erroneous. This prompted the following model \cite{DBLP:conf/icde/VerroiosG15} where the query answer (aggregated answer) is consistent (will not change on repeating the question), but the answer could be wrong with probability $p, 0 < p < 1/2$.  The value $p=\frac{1}{2}-\lambda$ with $\lambda >0$, implies that the answers are slightly better than random. The goal is to minimize the number of queries to reduce both cost and time of crowdsourcing, and recover the entities (clusters). This is exactly clustering with faulty oracle. While several heuristics have been developed \cite{DBLP:conf/icde/VerroiosG15,gokhale2014corleone,DBLP:conf/icde/VerroiosG17}, the field lacks rigorous theory with provable algorithms and hardness bounds. 
%Our framework also extends to the case where error probabilities may differ based on whether a pair-wise query refers to the same entity or not.

\vspace{0.1in}
\noindent{\bf Signed Edge Prediction.}
The edge sign prediction problem can be defined as follows. Suppose we are given a social network with signs on all its edges, but the sign from node $u$ to $v$, denoted by $s(u,v)\in \{\pm 1\}$ is hidden. The goal is to recover these signs as best as possible using minimal amount of information. Social interactions or sentiments can be both positive (``like'', ``trust'')  and negative (``dislike'', ``distrust''). \cite{leskovec2010predicting} provides several such examples; e.g., Wikipedia, where one can vote for or against the nomination of others to adminship \cite{burke2008mopping}, or Epinions and Slashdots where users can express trust or distrust, or can declare others to be friends or foes \cite{brzozowski2008friends,lampe2007follow}. Initiated by  \cite{cartwright1956structural,harary1953notion}, many techniques and related models using convex optimization, low-rank approximation and learning theoretic approaches have been used for this problem \cite{chiang2014prediction,cesa2012correlation,chen2014clustering}. Recently \cite{chen2012clustering,chen2014clustering,DBLP:journals/corr/MitzenmacherT16} proposed the following model for edge sign prediction. We can query a pair of nodes $(u,v)$ to test whether $s(u,v)=+1$ indicating $u$ and $v$ belong to the same cluster or $s(u,v)=-1$ indicating they are not. However, the query fails to return the correct answer with probability $0< p <1/2$, and we want to query the minimal possible pairs. This is exactly the case of  {\em clustering with noisy oracle}. Our result significantly improves, and generalizes over \cite{chen2012clustering,chen2014clustering,DBLP:journals/corr/MitzenmacherT16}.

\vspace{0.1in}
\noindent{\bf Correlation Clustering.}
In fact, when all pair-wise queries are given, and the goal is to recover the maximum likelihood (ML) clustering, then our problem is equivalent to {\em noisy correlation clustering} \cite{bbc:04,ms:10}. Introduced by \cite{bbc:04}, correlation clustering is an extremely well-studied model of clustering. We are given a graph $G=(V,E)$ with each edge $e \in E$ labelled either $+1$ or $-1$, the goal of correlation clustering is to either (a) minimize the number of disagreements, that is the number of intra-cluster $-1$ edges and inter-cluster $+1$ edges, or (b) maximize the number of agreements that is the number of intra-cluster $+1$ edges and inter-cluster $-1$ edges. Correlation clustering is NP-hard, but can be approximated well with provable guarantees \cite{bbc:04}. In a random noise model, also introduced by \cite{bbc:04} and studied further by \cite{ms:10}, we start with a ground truth clustering, and then each edge label is flipped with probability $p$. This is exactly the graph we observe if we make all possible pair-wise queries, and the ML decoding coincides with correlation clustering. The proposed algorithm of \cite{bbc:04} can recover in this case all clusters of size $\omega(\sqrt{|V|\log{|V|}})$, and if ``all'' the clusters have size $\Omega(\sqrt{|V|})$, then they can be recovered by \cite{ms:10}. Using our proposed algorithms for clustering with noisy oracle, we can also recover significantly smaller sized clusters given the number of clusters are not too many. Such a result is possible to obtain using the repeated-peeling technique of \cite{DBLP:conf/icml/AilonCX13}. However, our running time is significantly better. E.g. for $k \leq n^{1/6}$, we have a running time of $O(n\log{n})$, whereas for \cite{DBLP:conf/icml/AilonCX13}, it is dominated by the time to solve a convex optimization over $n$-vertex graph which is at least $O(n^3)$.

\vspace{0.1in}
\noindent{\bf Stochastic Block Model (SBM).} The clustering with faulty oracle is intimately connected with the 
 {\em planted partition model}, also known as the stochastic block model \cite{holland1983stochastic,dyer1989solution,decelle2011asymptotic,DBLP:conf/focs/AbbeS15,abh:16,hajek2015achieving,chin2015stochastic,mossel2015consistency}.
The stochastic block model is an extremely well-studied model of random graphs where two vertices within the same community share an edge with probability $p'$, and two vertices in different communities share an edge with probability $q'$. 
 It is often assumed that $k$, the number of communities, is a constant (e.g. $k=2$ is known as the {\em planted bisection model} and is studied extensively \cite{abh:16,mossel2015consistency,dyer1989solution} or a slowly growing function of $n$ (e.g. $k=o(\log{n})$). 
% The points are assigned to clusters according to a probability distribution indicating the relative sizes of the clusters. 
 There are extensive literature on characterizing the  threshold phenomenon in SBM in terms of the gap between $p'$ and $q'$ (e.g. see \cite{DBLP:conf/focs/AbbeS15} and therein for many references)  for exact and approximate recovery of clusters of nearly equal size\footnote{Most recent works consider the region of interest as $p'=\frac{a\log{n}}{n}$ and $q'=\frac{b\log{n}}{n}$ for some $a> b >0$.}. 
%  For $k=2$ and equal sized clusters, sharp thresholds are derived relatively recently in \cite{abh:16,mossel2015consistency} for a specific sparse range of $p'$ and $q'$ \footnote{Most recent works consider the region of interest as $p'=\frac{a\log{n}}{n}$ and $q'=\frac{b\log{n}}{n}$ for some $a> b >0$.}.  
 If we allow for different probability of errors for pairs of elements based on whether they belong to the same cluster or not, then the resultant faulty oracle model is an intriguing generalization of SBM. Consider the probability of error for a query on $(u,v)$ is $1-p'$ if $u$ and $v$ belong to the same cluster and $q'$ otherwise; but now, we can only learn a subset of the entries of an SBM matrix by querying adaptively. Understanding how the threshold of recovery changes for such an ``incomplete'' or ``space-efficient'' SBM will be a fascinating direction to pursue. In fact, our lower bound results extend to asymmetric probability values, while designing efficient algorithms and sharp thresholds are ongoing works. In \cite{chen2016community}, a locality model where measurements can only be obtained for nearby nodes is studied for two clusters with non-adaptive querying and allowing repetitions. It would also be interesting to extend our work with such locality constraints.

\vspace{0.1in}
\noindent{\bf Contributions.} Formally the {\em clustering with a faulty oracle} is defined as follows.
\begin{problem}[\cc]
Consider a set of points $V\equiv[n]$ containing $k$ latent clusters $V_i$, $i =1, \dots, k$, $V_i \cap V_j =\emptyset$, where $k$ and the subsets $V_i \subseteq [n]$  are unknown. There is an oracle $\mathcal{O}_{p,q}: V\times V \to \{\pm1\},$ with  two error parameters $p,q: 0< p<q< 1$. The oracle  takes as input a pair of vertices $u,v \in V \times V$, and  if $u,v$ belong to the same cluster then $\mathcal{O}_{p,q}(u,v)=+1$ with probability $1-p$ and $\mathcal{O}_{p,q}(u,v)=-1$ with probability $p$. On the other hand, if $u,v$ do not belong to the same cluster then $\mathcal{O}_{p,q}(u,v)=+1$ with probability $1-q$ and $\mathcal{O}_{p,q}(u,v)=-1$ with probability $q$. Such an oracle is called a {\em binary asymmetric channel}. A special case would be when $p =1-q = \frac12-\lambda, \lambda >0$, the binary {\em symmetric} channel, where the error rate is the same $p$ for all pairs. Except for the lower bound, we focus on the symmetric case in this paper. Note that the oracle returns the same answer on repetition.
Now, given $V$, find $Q \subseteq V \times V$ such that $|Q|$ is minimum, and from the oracle answers it is possible to recover $V_i$, $i=1,2,...,k$ with high probability\footnote{ high probability implies with probability $1-o_{n}(1)$, where $o_{n}(1) \rightarrow 0$ as $n \rightarrow \infty$}.
\vspace{0.03in}\\~
{\em~~ Our contributions are as follows.}
\end{problem}

$\bullet$ {\it Lower Bound (Section~\ref{sec:error-lc}).} We show that $\Omega(\frac{nk}{\Delta(p\|q)})$ is the information theoretic lower bound on the number of adaptive queries required to obtain the correct clustering with high probability even when the clusters are of similar size (see, Theorem \ref{thm:faulty}). Here $\Delta(p\|q)$ is the 
Jensen-Shannon divergence between Bernoulli $p$ and $q$ distributions.
%squared Hellinger divergence between Bernoulli $p$ and $q$ distributions, defined as $\cH^2(p\|q)=1-(\sqrt{pq}+\sqrt{(1-p)(1-q)}).$
%\frac{1}{\sqrt{2}}\sqrt{(\sqrt{p}-\sqrt{q})^2+(\sqrt{1-p}-\sqrt{1-q})^2}$.
 For the symmetric case, that is when $p=1-q$, 
 $\Delta(p\|1-p) = (1-2p)\log \frac{1-p}{p}$. In particular, if $p = \frac12-\lambda$, our lower bound on query complexity is $\Omega(\frac{nk}{\lambda^2})=\Omega(\frac{nk}{(1-2p)^2})$.
%$\cH^2(p\|1-p)=1-2\sqrt{p(1-p)}$.  
Developing lower bounds in the interactive setting especially with noisy answers appears to be significantly challenging as popular techniques based on Fano-type inequalities for multiple hypothesis testing \cite{CGT:12,Chen:14} do not apply, and we believe our technique will be useful in other noisy interactive learning settings.
\vspace{0.05in}\\~
$\bullet$ {\it Information-Theoretic Optimal Algorithm (Section~\ref{sec:faultyub}).} For the symmetric error case, we design an algorithm which asks at most $O(\frac{nk \log n}{(1-2p)^2})$ queries (Theorem \ref{theorem:cc-error}) matching the lower bound within an $O(\log{n})$ factor, whenever $p = \frac12-\lambda$.
%The same algorithm extends to the asymmetric case, with an additional factor of $\max\{ |\log \frac{1-q}{p}|,|\log \frac{1-p}{q}|\}$ in the number of queries. 
\vspace{0.05in}\\~
$\bullet$ {\it Computationally Efficient Algorithm (Section~\ref{sec:efficient}).} We next design an algorithm that is computationally efficient and runs in $O(n\log{n}+k^6)$ time and asks at most $O(\frac{nk^2 \log n}{(1-2p)^4})$ queries. Note that most prior works in SBM, or works on edge sign detection, only consider the case when $k$ is a constant \cite{DBLP:conf/focs/AbbeS15,hajek2015achieving,chin2015stochastic}, even just $k=2$ \cite{mossel2015consistency,abh:16,chen2012clustering,chen2014clustering,DBLP:journals/corr/MitzenmacherT16}. As long as, $k=O(n^{1/6})$, we get a running time of $O(n\log{n})$.
We can use this algorithm to recover all clusters of size at least $\min{(k,\sqrt{n})}\log{n}$ for 
correlation clustering on noisy graph, improving upon the results of \cite{bbc:04,ms:10}. The algorithm runs in time $O(n\log{n})$ whenever $k \leq n^{1/6}$, as opposed to $O(n^3)$ in  \cite{DBLP:conf/icml/AilonCX13}.
\vspace{0.05in}\\~
$\bullet$ {\it Nonadaptive Algorithm (Section~\ref{sec:na}).}
When the queries must be done up-front, for $k=2$, we give a simple $O(n \log{n})$ time algorithm that asks $O(\frac{n \log n}{(1-2p)^4})$ queries improving upon \cite{DBLP:journals/corr/MitzenmacherT16} where a polynomial time algorithm (at least with a running time of $O(n^3)$) is shown with number of queries $O(n\log{n}/(1/2-p)^{\frac{\log{n}}{\log{\log{n}}}})$ and over \cite{chen2012clustering,chen2014clustering} where $O(n {\rm poly}\log{n})$ queries are required under certain conditions on the clusters. Our result generalizes to $k > 2$, and we show interesting lower bounds in this setting.
Further, we derive new lower bound showing trade-off between queries and threshold of recovery for incomplete SBM in Sec.~~\ref{sec:nq}.

%% file: faulty-oracle-v5.tex
\newtheorem{conjecture}{Conjecture}

\section{Lower bound for the faulty-oracle model}
\label{sec:error-lc}

Note that we are not allowed to ask the same question multiple times to get the correct answer. In this case, even for probabilistic recovery, a minimum size bound on cluster size is required. 
%If there is no minimum size requirement on a cluster then the input graph can always consist of multiple clusters of very small size. 
For example, consider the following two different clusterings.
$C_1: V = \sqcup_{i=1}^{k-2} V_i \sqcup\{v_1,v_2\}\sqcup\{v_3\}$ and $C_2: V =   \sqcup_{i=1}^{k-2} V_i \sqcup\{v_1\}\sqcup\{v_2,v_3\}$. Now if one of these two clusterings are given two us uniformly at random, no matter how many queries we do, we will fail to recover the correct clustering with positive probability. Therefore, the  challenge in proving lower bounds is when clusters all have size more than a minimum threshold, or when they are all nearly balanced. % is nOur lower bound result works even when all the clusters are close to their average size (which is $\frac{n}{k}$). \textcolor{red}{and resolves a question from \cite{dkmr:14} for $p=0$ case. (Remove this?)}
This removes  the constraint on the algorithm designer on how many times a cluster can be queried with a vertex and the algorithms can have greater flexibility.
%While we have to show that enough number of queries must be made with a large number of vertices $V' \subset V$, either of the conditions on minimum or maximum sizes of a cluster ensures that $V'$ contains enough vertices that do not satisfy this query requirement.
 We define a clustering to be  {\em balanced} if either of the following two conditions 
hold 1)  the maximum size of a cluster is $\le \frac{4n}{k}$,  2) the minimum size of a cluster is $\ge \frac{n}{20k}$.
It is much harder to prove lower bounds if the clustering is balanced. %In fact, \cite{dkmr:14} gives a simple lower bound of $\Omega(nk)$ for randomized algorithms in the case of perfect oracle, but their lower bound does not work for the balanced case, which was explicitly mentioned  as an open question. 
\remove{
However, the argument above does not hold for the case when $p=1-q= 0$. In that case any (randomized) algorithm has to  use  (on expectation) $\Omega(nk)$ queries for correct clustering (e.g., see,  \cite{dkmr:14}). %Theorem \ref{thm:faulty} below where one can substitute $p=1-q=0$ as a special case).
 }

\remove{
While for deterministic algorithm the proof of the above fact is straight-forward, for randomized algorithms it  was established in \cite{dkmr:14}. %In \cite{dkmr:14}, a clustering was called {\em balanced} if the minimum and maximum sizes of the clusters are only a constant factor way (we have formally defined a balanced clustering in the introduction). 
In particular, \cite{dkmr:14} observes that, for unbalanced 
input the lower bound for $p=1-q=0$ case is easier. For balanced inputs, they left the lower bound for $p =1-q=0$ as an open problem. Theorem \ref{thm:faulty} resolves this as a special case.}

Our main lower bound in this section uses the Jensen-Shannon (JS) divergence. The well-known KL divergence is defined between two probability mass functions $f$ and $g$:
$
D(f\|g) = \sum_i f(i) \log \frac{f(i)}{g(i)}.
$
Further define the  JS divergence as: $\Delta(f\|g) = \frac12(D(f\|g) +D(g\|f))$.
% $$
% \cH^2(f\|g) = \frac{1}{2}\int \Big(\sqrt{f(x)}-\sqrt{g(x)}\Big)^2 dx,
% $$
% when $f$ and $g$ are pdfs, and 
% $
% \cH^2(f\|g) = \frac{1}{2}\sum_i \Big(\sqrt{f(i)}-\sqrt{g(i)}\Big)^2.
% $
In particular, the KL and JS divergences between two Bernoulli random variable with parameters $p$ and $q$ are denoted with $D(p\|q)$ and $\Delta(p\|q)$ respectively.

\begin{theorem}[\cc~Lower Bound] \label{thm:faulty}
Any (randomized) algorithm must make $\Omega\Big(\frac{nk}{\Delta(p\|q)}\Big)$ expected number of queries 
to recover  the correct clustering with  probability at least $\frac34$, even when the clustering is known to be balanced. %For $p =0$ or $q=1$, any (randomized) algorithm must make $\Omega(nk)$ queries
%to recover the correct clusters with  probability $0.9$  even when the clustering is known to be balanced. %(Proof in Sec.~\ref{sec:error-lc}).
\end{theorem}
Note that the lower bound is more effective when $p$ and $q$ are close. Moreover our actual lower bound is slightly tighter with 
the expected number of queries required  given by $\Omega\Big(\frac{nk}{\min\{D(q\|p),D(p\|q)\}}\Big).$

We have $V$ to be the $n$-element set to be clustered:  $V = \sqcup_{i=1}^k V_i$. 
%We also assume inputs are such that either 1) the maximum size of the cluster is within a constant times away from the average size, or 2) the minimum size of the cluster is  a constant fraction of the average size. Note that the average size of a cluster  is $\frac{n}{k}$.
To prove Theorem  \ref{thm:faulty}  % similar to  the one we have used in Theorem \ref{thm:lb-main}. 
we first show that, if the number of queries is small, then there exist $\Omega(k)$ number of clusters, that are not being sufficiently queried with. Then we show that, since the size of the clusters cannot be too large or too small, there exists a decent number of vertices in these clusters.
%However  
%we only handle binary random variables here (the answer to the queries).
%The significant difference is that, while designing the  input we consider a {\em balanced} clustering %do not restrict to  small fixed size clusters (
%with small sized clusters we can always fool any algorithm as 
%exemplified above). This removes  the constraint on the algorithm designer on how many times a cluster can be queried with a vertex.
%The main component of the proof,  Lemma \ref{lem:clus},  shows that enough number of queries must be made with a large number of vertices $V' \subset V$. Then either of the conditions on minimum or maximum sizes of a cluster ensures that $V'$ contains enough vertices that do not satisfy this query requirement.

 %as in the case of  the coin-tossing experiment.

%\begin{theorem} \label{thm:faulty}
%Assume either of the following cases:
%\begin{itemize}
%\item  the maximum size of a cluster is $\le \frac{4n}{k}$.
%\item   the minimum size of a cluster is $\ge \frac{n}{20k}$.
%\end{itemize}
%For a graph that satisfies either of the above two conditions, any (randomized) algorithm must make $\Omega\Big(\frac{nk}{D(p\|1-p)}\Big)$ queries
%to recover the correct clustering with  probability $0.9$ when $p >0$. For $p =0$ any (randomized) algorithm must make $\Omega(nk)$ queries
%to recover the correct clusters with  probability $0.9$. 
%\end{theorem}
  
The main piece of the proof of  Theorem  \ref{thm:faulty} is  Lemma \ref{lem:clus}. 

%We provide a sketch of this lemma here, the full proof is given in Appendix~\ref{ap:lem1}.
\begin{lemma}\label{lem:clus}
Suppose, there are $k$ clusters. There exist at least $\frac{4k}{5}$ clusters such that  an element $v$ from any one of these clusters will be assigned to a wrong cluster by any randomized algorithm with  probability $1/4$ unless the  total number  of queries involving $v$ is more than $\frac{k}{10\Delta(p\|q)}.$% when $p >0$ or $q <1$ and $\frac{k}{10}$ when $p =1-q=0$. %Then there vertex is assigned to a wrong cluster with probability $\frac{1}{5}.$
\end{lemma}

\begin{proof}
Our first task is to cast the problem as a hypothesis testing problem.

\vspace{0.1in}
\noindent{\bf Step 1: Setting up the hypotheses.}
Let us assume that the $k$ clusters are already formed, and we can moreover assume that all elements except for one element $v$ has already been assigned to a cluster. %Let $|V_1|  \le \frac{1}{D(p\|1-p)}.$
Note that, queries that do not involve the said element  plays no role in this stage.

Now the problem reduces to a hypothesis testing problem where the $i$th hypothesis 
$H_i$ 
for $i =1,\dots, k$, denotes that the true cluster for $v$ is $V_i$. 
We can also add a null-hypothesis $H_0$ that stands for the vertex belonging to none of the clusters (since $k$ is unknown this is a hypothetical possibility for any algorithm\footnote{this lower bound easily extend to the case even when $k$ is known}).
 Let $P_i$ denote the joint probability distribution of our observations (the answers to the queries involving vertex $v$) when $H_i$ is true, $i =1,\dots, k$. That is for any event $\cA$ we have,
$$
P_i(\cA) = \Pr(\cA | H_i).
$$

Suppose $T$   denotes the  total  number of queries made by a (possibly randomized) algorithm at this stage before assigning a cluster. % and $t = \avg T$. 
%involving vertex $v$. Also, let 
Also let $\underline{x}$ be the $T$ dimensional binary vector that is the result of the queries. The assignment is based on $\underline{x}$.
Let the random variable $T_i$ denote  the number of queries involving cluster $V_i, i =1,\dots, k.$
In the second phase, we need to identify a set of clusters that are not being queried with enough by the algorithm.

\vspace{0.1in}
\noindent{\bf Step 2: A set of ``weak'' clusters.}
%There must exist some $\ell \in [k]$, such that, %
We must have,
$
\sum_{i=1}^k \avg_0T_i = T.
$
Let, $$J_1 \equiv \{i \in \{1,\dots, k\}: \avg_0 T_i \le \frac{10T}k  \}.$$ Since,
$
(k -|J_1|) \frac{10T}k \le T,
$
we have $|J_1| \ge \frac{9k}{10}$. That is there exist at least $\frac{9k}{10}$ clusters in each of where less than $\frac{10T}k$ (on average under $H_0$) queries
were made before assignment.

%Furthermore, notice that for any $i\in J_1$,
%$P_0(T_i\ge \frac{30T}{k}) \le \frac13$.

Let $\cE_i \equiv \{ \text{ the algorithm outputs cluster } V_i\}$. Let $$J_2 = \{i \in \{1,\dots, k\}: P_0(\cE_i ) \le \frac{10}k\}.$$ Moreover, since $\sum_{i=1}^k P_0(\cE_i) \le 1$ we must have,
$
(k -|J_2|)\frac{10}k \le 1,
$
or $|J_2| \ge \frac{9k}{10}$. Therefore, $J = J_1 \cap J_2$ has size,
$$
|J| \ge 2 \cdot \frac{9k}{10} - k = \frac{4k}{5}.
$$

Now let us assume that, we are given an element $v \in V_j$ for some $j \in J$ to cluster ($H_j$ is the true hypothesis). The probability of correct clustering is $P_j(\cE_j)$. 
In the last step, we give an upper bound on probability of correct assignment for this element.

\vspace{0.1in}
\noindent{\bf Step 3: Bounding probability of correct assignment for weak cluster elements.}
We must have,
%\begin{align*}
%P_j(\cE_j) &\le P_j(\cE_j \mid T_j < \frac{30T}{k}) \\&+ P_j(\cE_j| T_j \ge \frac{30T}{k})P_j(T_j \ge \frac{30T}{k})\\
%&\le  P_0(\cE_j \mid T_j < \frac{30T}{k}) + P_j(\cE_j \mid T_j < \frac{30T}{k}) \\
%&- P_0(\cE_j \mid T_j < \frac{30T}{k}) + P_j(T_j \ge \frac{30T}{k})
%\end{align*}

\begin{align*}
P_j(\cE_j) &= P_0(\cE_j) + P_j(\cE_j) -P_0(\cE_j)\\
& \le  \frac{10}k + | P_0(\cE_j) -P_j(\cE_j)|\\
 &\quad \le \frac{10}k + \| P_0 -P_j\|_{TV}
% \le  \frac{10}k + \sqrt{ \frac{1}{2}D(P_0\|P_j)}.
  \le  \frac{10}k + \sqrt{ \frac{1}{2}D(P_0\|P_j)}.
 \end{align*}
where we again used the definition of the total variation distance and  in the last step we have used the Pinsker's inequality \cite{cover2012elements}. %(lemma \ref{lem:pinsker}).
The task is now to bound the divergence $D(P_0 \|P_j)$. Recall that $P_0$ and $P_j$ are the joint distributions of the 
independent random variables (answers to queries) that are identical to one of two Bernoulli random variables:$Y$, which is Bernoulli($p$), or
$Z$, which is Bernoulli($q$). Let $X_1,\dots,X_T$ denote the outputs of the queries, all independent random variables. %$w_{i,j}, i\in \cup_u V_u$.  
We must have, from the chain rule \cite{cover2012elements},
\begin{align*}
D(P_0 \|P_j) &= \sum_{i=1}^T %D(P_0(X_i) \| P_j(X_i))
D(P_0(x_i|x_1,\dots, x_{i-1}) \| P_j(x_i|x_1,\dots, x_{i-1})) \\
&=\sum_{i=1}^T \sum_{(x_1,\dots, x_{i-1}) \in \{0,1\}^{i-1}}P_0(x_1,\dots, x_{i-1})D(P_0(x_i|x_1,\dots, x_{i-1}) \| P_j(x_i|x_1,\dots, x_{i-1})).
\end{align*}
Note that, for the random variable $X_i$, the term 
%$D(P_0(X_i|x_1,\dots, x_{i-1}) \| P_j(X_i|x_1,\dots, x_{i-1}))$ 
%$D(P_0(X_i) \| P_j(X_i))$ 
$D(P_0(x_i|x_1,\dots, x_{i-1}) \| P_j(x_i|x_1,\dots, x_{i-1}))$
will contribute to $D(q\|p)$ only when the query involves the cluster $V_j$. Otherwise the term will contribute to $0$.  Hence, 
\begin{align*}
D(P_0 \|P_j)& = \sum_{i=1}^T \sum_{(x_1,\dots, x_{i-1}) \in \{0,1\}^{i-1}: i \text{th query involves } V_j}P_0(x_1,\dots, x_{i-1})D(q\|p)\\
& = D(q\|p)\sum_{i=1}^T \sum_{(x_1,\dots, x_{i-1}) \in \{0,1\}^{i-1}: i \text{th query involves } V_j}P_0(x_1,\dots, x_{i-1})\\
& = D(q\|p)\sum_{i=1}^T P_0( i \text{th query involves } V_j)
 = D(q\|p) \avg_0 T_j  \le \frac{10T}{k}D(q\|p).
\end{align*}
%Note that, for the random variable $X_i$, the term 
%%$D(P_0(X_i|x_1,\dots, x_{i-1}) \| P_j(X_i|x_1,\dots, x_{i-1}))$ 
%$D(P_0(X_i) \| P_j(X_i))$ will contribute to $D(p\|1-p)$ only when the query involves the cluster $V_j$. Otherwise the term will contribute to $0$. 
%Let $\chi_i\in \{0,1\}$ takes the value $1$ if and only if   the $i$th query involves $V_j$.
%Hence,
%\begin{align*}
%D(P_0 \|P_j)  = \sum_{i=1}^Q \avg_0\chi_i D(p\|1-p) = \avg_0 Q_j D(p\|1-p) \le \frac{10Q}{k}D(p\|1-p).
%\end{align*}
Now plugging this in,
  \begin{align*}
P_j(\cE_j)
 \le  \frac{10}k + \sqrt{ \frac{1}{2}\frac{10T}{k}D(q\|p)}
 \le  \frac{10}k + \sqrt{ \frac{1}{2}},
\end{align*}
if $T \le \frac{k}{10D(q\|p)}$ and large enough $k$. Had we bounded  the total variation distance with $D(P_j\|P_0)$ in the Pinsker's inequality then we would have $D(p\|q)$ in the denominator. Obviously the smaller of $D(p\|q)$ and $D(q\|p)$ would give the stronger lower bound. %On the other hand, when $p =0$, $P_j(\cE_j) <1$ when  $\avg_0 Q_j < 1$. Therefore $\frac{10Q}{k} \ge 1$
%whenever $p=0$.
 \end{proof}

Now we are ready to prove Theorem \ref{thm:faulty}. 
\vspace{-0.05in}
\begin{proof}[Proof of Theorem \ref{thm:faulty}]
We will show the claim by considering a balanced input. Recall that for a balanced input  either  the maximum size of a cluster is $\le \frac{4n}{k}$ or the minimum size of a cluster is $\ge \frac{n}{20k}$.   We will consider the two cases separately for the proof.

\noindent{\em Case 1: the maximum size of a cluster is $\le \frac{4n}{k}$.}

Suppose, the total number of queries is $T'$.  That means number of vertices involved in the queries is $\le 2T'$. Note that,  there are $k$ clusters and $n$ elements. Let $U$ be the set of vertices that are involved in less than $\frac{16T'}{n}$ queries. Clearly,
$
(n-|U|)\frac{16T'}{n} \le 2T',~\text{or } |U| \ge \frac{7n}8.$

Now we know from Lemma \ref{lem:clus} that there exists $\frac{4k}{5}$ clusters such that  a vertex $v$ from any one of these clusters will be assigned to a wrong cluster by any randomized algorithm with  probability $1/4$ unless the expected number  of queries involving this vertex is more than $\frac{k}{10\Delta(q\|p)}$. % for $p >0$ or $q <1$, and $\frac{k}{10}$ when $p =1-q=0$.

We claim that $U$ must have an intersection with at least one of these $\frac{4k}{5}$ clusters. If not, then  more than $\frac{7n}8$ vertices must belong to less than $k -\frac{4k}{5} = \frac{k}{5}$ clusters. Or the maximum size of a cluster will be $\frac{7n\cdot 5}{8k} >\frac{4n}{k},$
which is prohibited according to our assumption.

%Consider the case, $p >0$ or $q <1$. 
Now each vertex in the intersection of $U$ and the  $\frac{4k}{5}$ clusters are going to be assigned to an incorrect cluster with positive probability
if,
$
\frac{16T'}{n} \le \frac{k}{10\Delta(p\|q)}.
$
Therefore we must have 
$
T' \ge  \frac{nk}{160\Delta(p\|q)}.
$

\noindent{\em Case 2: the minimum size of a cluster is $\ge \frac{n}{20k}$.}

Let $U'$ be the set of clusters that are involved in at most $\frac{16T'}{k}$ queries. That means,
$
(k-|U'|)\frac{16T'}{k} \le 2T'.
$
This implies, $ |U'| \ge \frac{7k}{8}$. Now we know from Lemma \ref{lem:clus} that there exist $\frac{4k}{5}$ clusters (say $U^\ast$) such that  a vertex $v$ from any one of these clusters will be assigned to a wrong cluster by any randomized algorithm with  probability $1/4$ unless the expected number  of queries involving this vertex is more than $\frac{k}{10\Delta(p\|q)}$. % $p >0$ or $q <1$ and $\frac{k}{10}$ for $p =1-q =0$.
Quite clearly $|U^\ast \cap U| \ge \frac{7k}{8} + \frac{4k}{5} - k = \frac{27k}{40}$.

Consider a cluster $V_i$ such that $i \in U^\ast \cap U$, which is always possible because the intersection is nonempty.
$V_i$ 
  is involved in at most $\frac{16T'}{k}$ queries. Let the minimum size of any cluster be $t$. Now, at least half of the vertices of $V_i$ must each be involved in at most $\frac{32T'}{kt}$ queries. Now each of these vertices must be involved in at least $\frac{k}{10\Delta(p\|q)}$ queries (see  Lemma \ref{lem:clus}) to avoid being assigned to a wrong cluster with positive probability. % (for the case of $p =0$ this number would be $\frac{k}{10}$).
This means, $\frac{32T'}{kt} \ge \frac{k}{10\Delta(p\|q)}$
or $T' = \Omega\Big(\frac{nk}{\Delta(p\|q)}\Big),$
% for $p>0$, or $q <1$, 
since $t \ge \frac{n}{20k}$. %Similarly when $p =1-q=0$ we need $ T' = \Omega(nk)$.
\end{proof}

%% file: algo.tex
\section{Algorithms}\label{sec:faultyub}
In this section, we first develop an information theoretically optimal algorithm for clustering with faulty oracle within an $O(\log{n})$ factor of the optimal query complexity. Next, we show how the ideas can be extended to make it computationally efficient. We consider both the adaptive and non-adaptive versions. 

\subsection{Information-Theoretic Optimal Algorithm}
\label{sec:info-theory}

Let $V = \sqcup_{i=1}^k {V}_i$ be the true clustering and $V=\sqcup_{i=1}^k \hat{V}_i$ be the maximum likelihood (ML) estimate of  the clustering that can be found when all $\binom{n}2$ queries have been made to the faulty oracle. Our first result obtains a query complexity upper bound within an $O(\log{n})$ factor of the information theoretic lower bound. The algorithm runs in quasi-polynomial time, and we show this is the optimal possible assuming the famous {\em planted clique} hardness. Next, we show how the ideas can be extended to make it computationally efficient  in Section~\ref{sec:efficient}. We consider both the adaptive and non-adaptive versions. % The missing proofs and details are provided in the appendix. % in the supplementary document.

In particular, we prove the following theorem.
\begin{theorem}
\label{theorem:cc-error}
There exists an algorithm with query complexity $O(\frac{nk\log{n}}{(1-2p)^2})$ for \cc~that returns the ML estimate with high probability when query answers are incorrect with probability $p < \frac{1}{2}$. Moreover, the algorithm returns all true clusters of $V$ of size at least $\frac{C\log{n}}{(1-2p)^2}$ for a suitable constant $C$ with probability $1-o_{n}(1)$. 
\end{theorem}
\begin{remark}
Assuming $p=\frac{1}{2}-\lambda$, as $\lambda \to 0$, $\Delta(p\|1-p)=(1-2p) \ln \frac{1-p}{p} = 2\lambda\ln\frac{1/2+\lambda}{1/2 -\lambda} =2\lambda\ln(1+\frac{2\lambda}{1/2-\lambda}) \le \frac{4\lambda^2}{1/2-\lambda} = O(\lambda^2)=O((1-2p)^2)$, matching the query complexity lower bound within an $O(\log{n})$ factor. Thus our upper bound is within a $\log{n}$ factor of the information theoretic optimum in this range.
\end{remark}

\paragraph*{\bf Finding the Maximum Likelihood Clustering of $V$ with faulty oracle}
We can view the clustering  problem as following. We have an undirected graph $G(V\equiv[n],E)$, such that $G$ is a union of $k$ disjoint cliques $G_i(V_i, E_i)$, $i =1, \dots, k$. % and $G_2(V_2,E_2)$ of equal cardinality. 
 The subsets $V_i \in [n]$ are unknown to us; they are called the clusters of $V$.
 The adjacency matrix of $G$ is a block-diagonal matrix. Let us denote this matrix by $A = (a_{i,j})$.

Now suppose, each edge  of $G$ is erased  independently with probability $p$, and at the same time each non-edge is
replaced with an edge with probability $p$. Let the resultant adjacency matrix of the modified graph be 
$Z= (z_{i,j})$. The aim is to recover $A$ from $Z$.

%In proving Theorem~\ref{theorem:cc-error} and Theorem~\ref{lemma:correct} below, we will the fact that the ML estimate of $G'$ is given by 
%\begin{align}
%\label{eq:ml1}
%\max_{S_\ell, \ell = 1, \dots: V = \sqcup_{\ell=1} S_\ell}\sum_{\ell} \sum_{i, j \in S_\ell, i \ne j}\omega_{i,j},~~%\text{(see Lemma~\ref{lemma:ml} below)}
%\end{align}
%The proof of this fact is given in Lemma~\ref{lemma:ml} below.
%. We here give a proof.

\begin{lemma}
\label{lemma:ml}
The maximum likelihood recovery is given by the following:
\begin{align*}
\max_{S_\ell, \ell = 1, \dots: V = \sqcup_{\ell} S_\ell}& \prod_{\ell} \prod_{i, j \in S_\ell, i \ne j}P_+(z_{i,j}) \prod_{r,t, r\ne t} \prod_{i \in S_r, j\in S_t} P_-(z_{i,j})\\
= \max_{S_\ell, \ell = 1, \dots: V = \sqcup_{\ell=1} S_\ell}&\prod_{\ell} \prod_{i, j \in S_\ell, i \ne j}\frac{P_+(z_{i,j})}{P_-(z_{i,j})}
\prod_{i,j \in V, i \ne j} P_-(z_{i,j}).
\end{align*}
where, $P_+(1) =1-p, P_+(0) =p, P_-(1) =p, P_-(0) = 1-p.$
\end{lemma}
Therefore, the ML recovery asks for, 
$$
\max_{S_\ell, \ell = 1, \dots: V = \sqcup_{\ell=1} S_\ell}\sum_{\ell} \sum_{i, j \in S_\ell, i \ne j}\ln \frac{P_+(z_{i,j})}{P_-(z_{i,j})}.
$$
Note that, $$\ln \frac{P_+(0)}{P_-(0)} = - \ln \frac{P_+(1)}{P_-(1)} = \ln\frac{p}{1-p} .$$
Hence the ML estimation is,
\begin{align}
\label{eq:ml}
\max_{S_\ell, \ell = 1, \dots: V = \sqcup_{\ell=1} S_\ell}\sum_{\ell} \sum_{i, j \in S_\ell, i \ne j}\omega_{i,j},
\end{align}
where $\omega_{i,j} = 2z_{i,j} -1, i \ne j$, i.e., $\omega_{i,j} =1,$ when $z_{i,j} =1$ and $\omega_{i,j} =-1$ when $z_{i,j} =0, i \ne j$. Further $\omega_{i,i} = z_{i,i} =0, i = 1, \dots, n.$ We  will use this fact to prove Theorem~\ref{theorem:cc-error} and Theorem~\ref{lemma:correct} below.

 Note that \eqref{eq:ml} is equivalent to finding correlation clustering in $G$ with the objective of maximizing the consistency with the edge labels, that is we want to maximize the total number of positive intra-cluster edges and total number of negative inter-cluster edges \cite{bbc:04,ms:10,mmv:14}. This can be seen as follows.
 
  \begin{align*}
& \max_{S_\ell, \ell = 1, \dots: V = \sqcup_{\ell=1} S_\ell}\sum_{\ell} \sum_{i, j \in S_\ell, i \ne j}\omega_{i,j}\\
&\equiv \max_{S_\ell, \ell = 1, \dots: V = \sqcup_{\ell=1} S_\ell} \big[\sum_{\ell} \sum_{i, j \in S_\ell, i \ne j}\big|(i,j): \omega_{i,j}=+1\big|-\big|(i,j): \omega_{i,j}=-1\big|\big]+\sum_{i,j \in V, i \ne j}\big|(i,j): \omega_{i,j}=-1\big| \\
&=\max_{S_\ell, \ell = 1, \dots: V = \sqcup_{\ell=1} S_\ell} \big[\sum_{\ell} \sum_{i, j \in S_\ell, i \ne j}\big|(i,j): \omega_{i,j}=+1\big|+\big[\sum_{r,t: r \ne t} \big|(i,j): i \in S_r, j \in S_t, \omega_{i,j}=-1\big|\big].
\end{align*}
Therefore \eqref{eq:ml} is same as correlation clustering. However going forward we will be viewing it as obtaining clusters with maximum intra-cluster weight. That will help us to obtain the desired running time of our algorithm. Also, note that, we have a random instance of correlation clustering here, and not a worst case instance.

\vspace{0.1in}
\paragraph*{\bf Algorithm. 1} The algorithm that we propose  has several phases. The main idea is as follows. We start by selecting a small subset of vertices, and extract the heaviest weight subgraph in it by suitably defining edge weight. If the subgraph extracted has $\sim \log{n}$ size, we are confident that it is part of an original cluster. We then grow it completely, where a decision to add a new vertex to it happens by considering the query answers involving these different $\log{n}$ vertices and the new vertex. Otherwise, if the subgraph extracted has size less than $\log{n}$, we select more vertices. We note that we would never have to select more than $O(k\log{n})$ vertices, because by pigeonhole principle, this will ensure that we have selected at least $\sim \log{n}$ members from a cluster, and the subgraph detected will have size at least $\log{n}$. This helps us to bound the query complexity. We emphasize that our algorithm is completely deterministic.

\vspace{0.1in}
\noindent{\it Phase 1: Selecting a small subgraph.}
Let $c=\frac{16}{(1-2p)^2}$. 

\begin{enumerate}[noitemsep]
\item Select $c\log{n}$ vertices arbitrarily from $V$. Let $V'$ be the set of selected vertices. Create a subgraph $G'=(V',E')$ by querying for every $(u,v) \in V' \times V'$ and assigning a weight of $\omega(u,v)=+1$ if the query answer is ``yes'' and $\omega(u,v)=-1$ otherwise . 
\item Extract the heaviest weight subgraph $S$ in $G'$. If $|S| \geq c\log{n}$, move to Phase 2.
\item Else we have $|S|< c\log{n}$. Select a new vertex $u$, add it to $V'$, and query $u$ with every vertex in $V'\setminus \{u\}$. Move to step (2).
\end{enumerate}

\noindent{\it Phase 2: Creating an Active List of Clusters.} Initialize an empty list called $\sf{active}$ when Phase 2 is executed for the first time.

\begin{enumerate}[noitemsep]
\item Add $S$ to the list $\sf{active}$.
\item Update $G'$ by removing $S$ from $V'$ and every edge incident on $S$. For every vertex $z \in V'$, if $\sum_{u \in S} \omega{(z,u)} > 0$, include $z$ in $S$ and remove $z$ from $G'$ with all edges incident to it. 
\item Extract the heaviest weight subgraph $S$ in $G'$. If $|S| \geq c\log{n}$, Move to step(1). Else move to Phase $3$.
\end{enumerate}

\noindent{\it Phase 3: Growing the Active Clusters.} We now have a set of clusters in ${\sf active}$. 

\begin{enumerate}[noitemsep]
\item Select an unassigned vertex $v$ not in $V'$ (that is previously unexplored), and for every cluster $\cC \in \sf{active}$, pick $c\log{n}$ distinct vertices $u_1, u_2,...., u_l$ in the cluster and query $v$ with them. If the majority of these answers are ``yes'', then include $v$ in $\cC$. 
\item Else we have for every $\cC \in \sf{active}$ the majority answer is ``no'' for $v$.  Include $v \in V'$ and query $v$ with every node in $V' \setminus {v}$ and update $E'$ accordingly. Extract the heaviest weight subgraph $S$ from $G'$ and if its size is at least $c\log{n}$ move to Phase 2 step (1). Else move to Phase 3 step (1) by selecting another unexplored vertex.
\end{enumerate}

\noindent{\it Phase 4: Maximum Likelihood (ML) Estimate.}

\begin{enumerate}[noitemsep]
\item When there is no new vertex to query in Phase $3$, extract the maximum likelihood clustering of $G'$ and return them along with the active clusters, where the ML estimation is defined in Equation~\ref{eq:ml}. % \ref{appendix:faulty})

\end{enumerate}

\remove{\begin{remark}
Lemma \ref{lemma:ml} (Appendix~\ref{appendix:faulty}) shows that \eqref{eq:ml1} is the ML estimate when all query answers are given. Phase $4$ is not required if all the clusters have size at least $c\log{n}$. Moreover, the ML estimate may not be same as the true clustering if there exists clusters of size $o(\log_{p}{n})$--all queries involving a vertex in such a cluster and other cluster members could be $-1$ with probability at least $\frac{1}{o(n)}$. 

%Note that \eqref{eq:ml1} is equivalent to finding correlation clustering in $G$ with the objective of maximizing the consistency with the edge labels, that is we want to maximize the total number of positive intra-cluster edges and total number of negative inter-cluster edges \cite{bbc:04,ms:10,mmv:14}. See Appendix  for details. %~\ref{appendix:faulty}
\end{remark}}

\paragraph*{\bf Analysis.} 
%The high level steps of the analysis are as follows. Suppose all $\binom{n}{2}$ queries on $V \times V$  have been made. If the ML estimate of the clustering with these $\binom{n}{2}$  answers is same as the true clustering of $V$ that is, $\sqcup_{i=1}^k {V}_i \equiv \sqcup_{i=1}^k \hat{V}_i$ then the algorithm for noisy oracle finds the true clustering with high probability. 
%
%Let without loss of generality, $|\hat{V}_1| \geq ...\geq |\hat{V}_l| \geq 6c\log{n} > |\hat{V}_{l+1}| \geq...\geq |\hat{V}_k|$. We will show that Phase $1$-$3$ recover $\hat{V}_1,\hat{V}_2 ... \hat{V}_l$ with probability at least $1-\frac{1}{n}$. The remaining clusters are recovered in Phase $4$.
%
%\remove{Note that,
%$\cH^2(p \|1-p) =(\sqrt{p}-\sqrt{1-p})^2 \le \frac{(\sqrt{p}-\sqrt{1-p})^2}{2p}$, as $p \le 1/2$, and $(1-2p)^2 = (1-p -p)^2 = (\sqrt{p}-\sqrt{1-p})^2(\sqrt{p}+\sqrt{1-p})^2 \ge  (\sqrt{p}-\sqrt{1-p})^2(p+1-p)^2 =  (\sqrt{p}-\sqrt{1-p})^2$. }
%
%A subcluster is a subset of nodes in some cluster. Lemma \ref{lemma:mlG'} shows that any set $S$ that is included in ${\sf active}$ in Phase $2$ of the algorithm is a subcluster of $V$. This establishes that all clusters in ${\sf active}$ at any time are subclusters of some original cluster in $V$. 
To establish the correctness of the algorithm, we show the following. Suppose all $\binom{n}{2}$ queries on $V \times V$  have been made. If the ML estimate of the clustering with these $\binom{n}{2}$  answers is same as the true clustering of $V$ that is, $\sqcup_{i=1}^k {V}_i \equiv \sqcup_{i=1}^k \hat{V}_i$ then the algorithm for faulty oracle finds the true clustering with high probability. 

Let without loss of generality, $|\hat{V}_1| \geq ...\geq |\hat{V}_l| \geq 6c\log{n} > |\hat{V}_{l+1}| \geq...\geq |\hat{V}_k|$. We will show that Phase $1$-$3$ recover $\hat{V}_1,\hat{V}_2 ... \hat{V}_l$ with probability at least $1-\frac{1}{n}$. The remaining clusters are recovered in Phase $4$.

\remove{Note that,
$\cH^2(p \|1-p) =(\sqrt{p}-\sqrt{1-p})^2 \le \frac{(\sqrt{p}-\sqrt{1-p})^2}{2p}$, as $p \le 1/2$, and $(1-2p)^2 = (1-p -p)^2 = (\sqrt{p}-\sqrt{1-p})^2(\sqrt{p}+\sqrt{1-p})^2 \ge  (\sqrt{p}-\sqrt{1-p})^2(p+1-p)^2 =  (\sqrt{p}-\sqrt{1-p})^2$. }

A subcluster is a subset of nodes in some cluster. Lemma \ref{lemma:mlG'} shows that any set $S$ that is included in ${\sf active}$ in Phase $2$ of the algorithm is a subcluster of $V$. This establishes that all clusters in ${\sf active}$ at any time are subclusters of some original cluster in $V$. Next, Lemma \ref{lemma:vertex} shows that elements that are added to a cluster in ${\sf active}$ are added correctly, and no two clusters in ${\sf active}$ can be merged. Therefore, clusters obtained from ${\sf active}$ are the true clusters. Finally, the remaining of the clusters can be retrieved from $G'$ by computing a ML estimate on $G'$ in Phase $4$, leading to Theorem \ref{lemma:correct}.

We will use the following version of the Hoeffding's inequality heavily in our proof. We state it here for the sake of completeness.

Hoeffding's  inequality for large deviation of sums  of bounded independent random variables is well known \cite{hoeffding1963probability}[Thm. 2].
\begin{lemma}[Hoeffding]\label{lem:hoef1}
If $X_1, \dots, X_n$ are  independent random variables   and $a_i\le X_i\le b_i$ for all $i\in [n].$ Then
$$
\Pr(|\frac1n\sum_{i=1}^n (X_i - \avg X_i) | \ge t) \le 2 \exp(-\frac{2n^2t^2}{\sum_{i=1}^n (b_i-a_i)^2}). 
$$
\end{lemma}
This inequality can be used when the random variables are independently sampled with replacement from a finite sample space.  
However due to a result in the same paper  \cite{hoeffding1963probability}[Thm. 4], this inequality also holds when the random variables are sampled
without replacement from a finite population.
\begin{lemma}[Hoeffding]\label{lem:hoef2}
If $X_1, \dots, X_n$ are  random variables  sampled without replacement from a finite set $\cX \subset \reals$, and $a\le x\le b$ for all $x\in \cX.$ Then
$$
\Pr(|\frac1n\sum_{i=1}^n (X_i - \avg X_i) | \ge t) \le 2 \exp(-\frac{2nt^2}{(b-a)^2}). 
$$
\end{lemma}

\begin{lemma}
\label{lemma:mlG'}
Let $c'=6c=\frac{96}{(1-2p)^2}$. %where $\lambda=\frac{1}{2}-p$. 
Algorithm $1$ in Phase $1$ and $3$ returns a subcluster of $V$ of size at least $c\log{n}$ with high probability if $G'$ contains a subcluster of $V$ of size at least $c'\log{n}$. Moreover, it does not return any set of vertices of size at least $c\log{n}$ if $G'$ does not contain a subcluster of $V$ of size at least $c\log{n}$.
\end{lemma}
\begin{proof}
Let $V'=\bigcup V'_i$, $i\in [1,k]$,  $V'_i \cap V'_j =\emptyset$ for $i \neq j$, and $V'_i \subseteq V_i$. Suppose without loss of generality $|V'_1| \geq |V'_2| \geq ....\geq |V'_k|$.
The lemma is proved via a series of claims. The proofs of the claims are delegated to Appendix~\ref{appen:faultyub}.
\begin{claim}
\label{claim:0}
Let $|V'_1| \geq c'\log{n}$. Then a set $S \subseteq V_i$ for some $i \in [1,k]$ will be returned with high probability when $G'$ is processed.
\end{claim}

%Lemma~\ref{lemma:mlG'} is proven in three steps. Step 1 shows that if $V'$ contains a subcluster of size $\geq c'\log{n}$ then $S \subseteq V_i$ for some $i \in [1,k]$ will be returned with high probability when $G'$ is processed. Step 2 shows that size of $S$ will be at least $c\log{n}$, and finally step 3 shows that if there is no subcluster of size at least $c\log{n}$ in $V'$, then no subset of size $> c\log{n}$ will be returned by the algorithm when processing $G'$, because otherwise that $S$ will span more than one cluster, and the weight of a subcluster contained in $S$ will be higher than $S$ giving to a contradiction.
%
%From Lemma \ref{lemma:mlG'}, any $S$ added to ${\sf active}$ in Phase $2$ is a subcluster with high probability, and has size at least $c\log{n}$. Moreover, whenever $G'$ contains a subcluster of $V$ of size at least $c'\log{n}$, it is retrieved by the algorithm and added to ${\sf active}$. The next lemma shows that each subcluster added to ${\sf active}$ is correctly grown to the true cluster: (1) every vertex added to such a cluster is correct, and (2) no two clusters in ${\sf active}$ can be merged. Therefore, clusters obtained from ${\sf active}$ are the true clusters. 
\begin{claim}
\label{claim:1}
Let $|V'_1| \geq c'\log{n}$. Then a set $S \subseteq V_i$ for some $i \in [1,k]$ with size at least $c\log{n}$ will be returned with high probability when $G'$ is processed.
\end{claim}

%We know $|S| \ge c' \sqrt{\frac{2(1-2p)}{3}} \log n$, while $|V_1'| \geq c'\log{n}$. In fact, with high probability, $|S| \geq \frac{(1-\delta)}{2}c'\log{n}$. Since all the vertices in $S$ belong to the same cluster in $V$, this holds again by the application of Hoeffding's inequality. Otherwise, the probability that the weight of $S$ is at least as high as the weight of $V_1'$ is at most $\frac{1}{n^2}$.

\begin{claim}
\label{claim:2}
If $|V'_1| < c\log{n}$. then no subset of size $> c\log{n}$ will be returned by the algorithm for faulty oracle when processing $G'$ with high probability. 
\end{claim}

Since, the algorithm attempts to extract a heaviest weight subgraph at most $n$ times, and each time the probability of failure is at most $O(\frac{1}{n^2})$. By union bound, all the calls succeed with probability at least $1-O(\frac{1}{n})$. This establishes the lemma.
\end{proof}

We will need the following version of Chernoff bound as well.
\begin{lemma}[Chernoff Bound]
\label{lemma:chernoff}
Let $X_1, X_2,...,X_n$ be independent binary random variables, and $X=\sum_{i=1}^{n}X_i$ with $E[X]=\mu$. Then for any $\epsilon > 0$
\[\Pr[X \geq (1+\epsilon) \mu] \leq \exp\Big(-\frac{\epsilon^2}{2+\epsilon}\mu\Big)\]
and, 
\[\Pr[X \leq (1-\epsilon)\mu] \leq \exp\Big(-\frac{\epsilon^2}{2}\mu\Big)\]
\end{lemma}

\begin{lemma}
\label{lemma:vertex}
 The list ${\sf active}$ contains all the true clusters of $V$ of size $\geq c'\log{n}$ at the end of the algorithm with high probability.
\end{lemma}

\begin{proof}
From Lemma \ref{lemma:mlG'}, any cluster that is added to ${\sf active}$ in Phase $2$ is a subset of some original cluster in $V$ with high probability, and has size at least $c\log{n}$. Moreover, whenever $G'$ contains a subcluster of $V$ of size at least $c'\log{n}$, it is retrieved by the algorithm and added to ${\sf active}$.

When a vertex $v$ is added to a cluster $\cC$ in ${\sf active}$, we have $|\calC| \geq c\log{n}$ at that time, and there exist $l=c\log{n}$ distinct members of $\calC$, say, $u_1,u_2,..,u_l$ such that majority of the queries of $v$ with these vertices returned $+1$. Let if possible $v \not\in \cC$. Then the expected number of queries among the $l$ queries that had an answer ``yes'' (+1) is $lp$. We now use the Chernoff bound, Lemma~\ref{lemma:chernoff} bound, to have,
 $$\Pr(v \text{ added to } \cC \mid v \not\in \cC) \leq e^{-lp\frac{(\frac{1}{2p}-1)^2}{2+(\frac{1}{2p}-1)}}\leq \frac{1}{n^3}.$$

%By the standard Chernoff-Hoeffding bound, $\Pr(v \notin \calC) \leq \text{exp}(-c\log{n}\frac{(1-2p)^2}{12p})=\text{exp}(-\frac{\log{n}}{p})\leq \frac{1}{n^2},$ as $p < \frac{1}{2}$.
%
%$\text{exp}(-c\log{n}\frac{2\lambda^2}{3(1+2\lambda)})\leq \text{exp}(-c\log{n}\frac{\lambda^2}{3})$, where the last inequality followed since $\lambda < \frac{1}{2}$. 

On the other hand, if there exists a cluster $\calC \in {\sf active}$ such that $v \in \calC$, then while growing $\cC$, $v$ will be added to $\cC$ (either $v$ already belongs to $G'$, or is a newly considered vertex). This again follows by the Chernoff bound. Here the expected number of queries to be answered ``yes'' is $(1-p)l$. Hence the probability that less than $\frac{l}{2}$ queries will be answered yes is $\Pr(v \text{ not included in } \calC \mid v \in \calC) \leq \text{exp}(-c\log{n}(1-p)\frac{(1-2p)^2}{8(1-p)^2})=\text{exp}(-\frac{2}{(1-p)}\log{n})\leq \frac{1}{n^2}$. Therefore, for all $v$, if $v$ is included in a cluster in ${\sf active}$, the assignment is correct with probability at least $1-\frac{1}{n}$. Also, the assignment happens as soon as such a cluster is formed in ${\sf active}$ and $v$ is explored (whichever happens first).

Furthermore, two clusters in ${\sf active}$ cannot be merged. Suppose, if possible there are two clusters $\calC_1$ and $\calC_2$ which ought to be subset of the same cluster in $V$. Let without loss of generality $\calC_2$ is added later in ${\sf active}$. Consider the first vertex $v \in \calC_2$ that is considered by our algorithm. If $\calC_1$ is already there in ${\sf active}$ at that time, then with high probability $v$ will be added to $\calC_1$ in Phase $3$. Therefore, $\calC_1$ must have been added to ${\sf active}$ after $v$ has been considered by our algorithm and added to $G'$. Now, at the time $\calC_1$ is added to $A$ in Phase $2$, $v \in V'$, and again $v$ will be added to $\calC_1$ with high probability in Phase $2$--thereby giving a contradiction.

This completes the proof of the lemma.
\end{proof}

\begin{theorem}\label{lemma:correct}
If the ML estimate of the clustering of $V$ with all possible $\binom{n}{2}$ queries return the true clustering, then the algorithm for faulty oracle returns the true clusters with high probability. Moreover, it returns all the true clusters of $V$ of size at least $c'\log{n}$ with high probability.
\end{theorem}
\begin{proof}
From Lemma \ref{lemma:mlG'} and Lemma \ref{lemma:vertex}, ${\sf active}$ contains all the true clusters of $V$ of size at least $c'\log{n}$ with high probability. Any vertex that is not included in the clusters in ${\sf active}$ at the end of the algorithm, are in $G'$. Also $G'$ contains all possible pairwise queries among them. Clearly, then the ML estimate of $G'$ will be the true ML estimate of the clustering restricted to these clusters.
\end{proof}

Finally, once all the clusters in ${\sf active}$ are grown, we have a fully queried graph in $G'$ containing the small clusters which can be retrieved in Phase $4$. This completes the correctness of the algorithm. With the following lemma, we get Theorem~\ref{theorem:cc-error}.
%\paragraph*{Query complexity of the faulty oracle algorithm.}
\begin{lemma}
\label{lemma:query}
%Let $p = \frac12 -\lambda$. 
The query complexity of the algorithm for faulty oracle is $O\Big(\frac{nk\log{n}}{(1-2p)^2}\Big)$.
\end{lemma}
\begin{proof}
Let there be $k'$ clusters in ${\sf active}$ when $v$ is considered by the algorithm. $k'$ could be $0$ in which case $v$ is considered in Phase $1$, else $v$ is considered in Phase $3$. Therefore, $v$ is queried with at most $ck'\log{n}$ members, $c\log{n}$ each from the $k'$ ${\sf active}$ clusters. If $v$ is not included in one of these clusters, then $v$ is added to $G'$ and queried will all vertices $V'$ in $G'$. We have seen in the correctness proof (Lemma \ref{lemma:correct}) that if $G'$ contains at least $c'\log{n}$ vertices from any original cluster, then ML estimate on $G'$ retrieves those vertices as a cluster with high probability. Hence, when $v$ is queried with the vertices in $G'$, $|V'|\leq (k-k')c'\log{n}$. Thus the total number of queries made when the algorithm considers $v$ is at most $c'k\log{n}$, where $c'=6c=\frac{96}{(2p-1)^2}$ when the error probability is $p$. This gives the query complexity of the algorithm considering all the vertices, which matches the lower bound computed in Section \ref{sec:error-lc} within an $O(\log{n})$ factor. % since $D(p\|1-p) = (1-2p) \ln \frac{1-p}{p} = 2\lambda\ln\frac{1/2+\lambda}{1/2 -\lambda} =2\lambda\ln(1+\frac{2\lambda}{1/2-\lambda}) \le \frac{4\lambda^2}{1/2-\lambda} = O(\lambda^2)$.
\end{proof}

Now combining all these we get the statement of Theorem \ref{theorem:cc-error}.
%\begin{theorem}
%\label{theorem:cc-error}
%There exists an algorithm with query complexity $O(\frac{1}{\lambda^2}nk\log{n})$ for \cc~when query answers are incorrect with probability $\frac{1}{2}-\lambda$. This matches the information theoretic lower bound on the query complexity within a $\log{n}$ factor.
%\end{theorem}

%Running time of this algorithm is dominated by finding the heaviest weight subgraph in $G'$, execution of each of those calls can be done in time $O([\frac{k\log{n}}{(2p-1)^2}]^{O(\frac{\log{n}}{(2p-1)^2})})$, that is quasi-polynomial in $n$. We show that it is unlikely that this running time can be improved by showing a reduction from the famous {\em planted clique problem} for which quasi-polynomial time is the best known (see Appendix).
\vspace{0.2in}
\paragraph*{\bf Running Time \& Connection to Planted Clique}
\label{sec:hardness}
While the algorithm described above is very close to information theoretic optimal, the running time is not polynomial. Moreover, it is unlikely that the algorithm can be made efficient. 

A crucial step of our algorithm is to find a large cluster of size at least $O(\frac{\log{n}}{(2p-1)^2})$, which can of course be computed in $O(n^{\frac{\log{n}}{(2p-1)^2}})$ time. However, since size of $G'$ is bounded by $O(\frac{k\log{n}}{(2p-1)^2})$, the running time to compute such a heaviest weight subgraph is $O([\frac{k\log{n}}{(2p-1)^2}]^{\frac{\log{n}}{(2p-1)^2}})$. This running time is unlikely to be improved to a polynomial. This follows from the planted clique conjecture.

\begin{conjecture}[Planted Clique Hardness]
Given an Erd\H{o}s-R\'{e}nyi random graph $G(n,p)$, with $p=\frac{1}{2}$, the planted clique conjecture states that if we plant in $G(n,p)$ a clique of size $t$ where $t=[\Omega(\log{n}), o(\sqrt{n})]$, then there exists no polynomial time algorithm to recover the largest clique in this planted model.
\end{conjecture}

{\bf Reduction.} Given such a graph with a planted clique of size $t=\Theta(\log{n})$, we can construct a new graph $H$ by randomly deleting each edge with probability $\frac{1}{3}$. Then in $H$, there is one cluster of size $t$ where edge error probability is $\frac{1}{3}$ and the remaining clusters are singleton with inter-cluster edge error probability being $\frac{1}{2}*\frac{2}{3}=\frac{1}{3}$. So, if we can detect the heaviest weight subgraph in polynomial time in the faulty oracle algorithm, then there will be a polynomial time algorithm for the planted clique problem.

In fact, the reduction shows that if it is computationally hard to detect a planted clique of size $t$ for some value of $t >0$, then it is also computationally hard to detect a cluster of size $\leq t$ in the faulty oracle model. Note that $t=o(\sqrt{n})$. In the next section, we propose a computationally efficient algorithm which recovers all clusters of size at least $\frac{\min{(k,\sqrt{n})}\log{n}}{(1-2p)^2}$ with high probability, which is the best possible assuming the conjecture, and can potentially recover much smaller sized clusters if $k=o(\sqrt{n})$.

 \remove{In this section we propose a computationally efficient algorithm which recovers all clusters of size at least $\min{(k,\sqrt{n})}\frac{\log{n}}{\cH(p\|(1-p))^2}=\min{(k,\sqrt{n})}\frac{\log{n}}{(1-2p)^2}$ with high probability. Again, this is the best possible assuming the hardness of planted clique. In fact, the reduction in the previous section showed that if it is computationally hard to detect a planted clique of size $t$ for some value of $t >0$, then it is also computationally hard to detect a cluster of size $\leq t$ in the faulty oracle model. Now noting that $t=\sqrt{n}$ for an efficient polynomial time algorithm, the tightness of the result follows. However, indeed, if $k=o(\sqrt{n})$, then we can recover all clusters of size $\geq \frac{k\log{n}}{(1-2p)^2}$.}

\subsection{Computationally Efficient Algorithm}
\label{sec:efficient}
%We now prove the following theorem. We give the algorithm here which is completely deterministic with known $k$-the extension to unknown $k$ and a detailed proof of correctness are deferred to Appendix~\ref{appen:efficient}.
 \paragraph*{\bf Known $k$}
We first design an algorithm when $k$, the number of clusters is known. Then we extend it to the case of unknown $k$. The algorithm is completely deterministic. 
 
 \begin{theorem}
\label{cor:er-poly}
 There exists a polynomial time algorithm with query complexity $O(\frac{nk^2}{(2p-1)^4})$ for \cc~with error probability $p$ and known $k$, that recovers all clusters of size at least $\Omega(\frac{k\log n}{(2p-1)^4})$. % in $G$.
 \end{theorem}
The algorithm is given below.
\remove{In this section we propose a computationally efficient algorithm which recovers all clusters of size at least $\min{(k,\sqrt{n})}\frac{\log{n}}{\cH(p\|(1-p))^2}=\min{(k,\sqrt{n})}\frac{\log{n}}{(1-2p)^2}$ with high probability. Again, this is the best possible assuming the hardness of planted clique. In fact, the reduction in the previous section showed that if it is computationally hard to detect a planted clique of size $t$ for some value of $t >0$, then it is also computationally hard to detect a cluster of size $\leq t$ in the faulty oracle model. Now noting that $t=\sqrt{n}$ for an efficient polynomial time algorithm, the tightness of the result follows. However, indeed, if $k=o(\sqrt{n})$, then we can recover all clusters of size $\geq \frac{k\log{n}}{(1-2p)^2}$.

\paragraph*{Known $k$}
We first design an algorithm when $k$, the number of clusters is known. Then we extend it to the case of unknown $k$. The algorithm is completely deterministic. To keep the analysis and the algorithm clean, we hide the constants.}

\noindent{\bf Algorithm 2.} Let $N=\frac{64k^2\log{n}}{(1-2p)^4}$. We define two thresholds $T(a)=pa+\frac{6}{(1-2p)}\sqrt{N\log{n}}$ and $\theta(a)=2p(1-p)a+2\sqrt{N\log{n}}$. The algorithm is as follows.

\vspace{0.1in}
\noindent{\it Phase 1-2C: Select a Small Subgraph.} Initially we have an empty graph $G'=(V',E')$, and all vertices in $V$ are unassigned to any cluster.
\begin{enumerate}[noitemsep,leftmargin=*]
\item Select $X$ new vertices arbitrarily from the unassigned vertices in $V\setminus V'$ and add them to $V'$ such that the size of $V'$ is $N$. If there are not enough vertices left in $V\setminus V'$, select all of them. Update $G'=(V',E')$ by querying for every $(u,v)$ such that $u \in X$ and $v \in V'$ and assigning a weight of $\omega(u,v)=+1$ if the query answer is ``yes'' and $\omega(u,v)=-1$ otherwise . 
\item Let $N^+(u)$ denote all the neighbors of $u$ in $G'$ connected by $+1$-weighted edges. We now cluster $G'$. Select every $u$ and $v$ such that $u \neq v$ and $|N^+(u)|, |N^+(v)| \geq T(|V'|)$. Then if $|N^+(u)\setminus N^+(v)|+|N^+(v)\setminus N^+(u)| \leq \theta(|V'|)$ (the symmetric difference of these neighborhoods) include $u$ and $v$ in the same cluster. Include in ${\sf active}$ all clusters formed in this step that have size at least $\frac{64k\log{n}}{(1-2p)^4}$. If there is no such cluster, abort. Remove all vertices in such cluster from $V'$ and any edge incident on them from $E'$.
\end{enumerate}

\noindent{\it Phase 3C: Growing the Active Clusters.} 
\begin{enumerate}[noitemsep,leftmargin=*]
\item For every unassigned vertex $v \in V \setminus V'$, and for every cluster $\cC \in \sf{active}$, pick $c\log{n}$ distinct vertices $u_1, u_2,...., u_l$ in the cluster and query $v$ with them. If the majority of these answers are ``yes'', then include $v$ in $\cC$. 
\item Output all the clusters in $\sf{active}$ and move to Phase 1 step (1) to obtain the remaining clusters.
\end{enumerate}
%Running time of the algorithm can be shown to be $O(nk\log{n}+kN^\omega)$ where $\omega \leq 2.373$ is the exponent of fast matrix multiplication\footnote{fast matrix multiplication can be avoided by slightly increasing the dependency on $k$}. Thus for small values of $k$, we get a highly efficient algorithm.  Moreover, using the algorithm for unknown $k$ verbatim (see below), we can obtain a correlation clustering algorithm for random noise model that recovers all clusters of size $\Omega(\frac{\min(k,\sqrt{n})\log{n}}{(2p-1)^4})$, improving over \cite{bbc:04,ms:10} for $k< \frac{\sqrt{n}}{\log{n}}$ since our ML estimate on $G'$ is correlation clustering. 

\vspace{0.2in}
\noindent{\bf Analysis.} Note that at every iteration, we consider a set  $X$ of new vertices from $V \setminus V'$ which have not been previously included in any cluster considered in ${\sf active}$, and query all pairs in $X \times V'\setminus V$. Let $A$ denote the fixed $n \times n$ matrix, where if $(i,j), i ,j \in V$ is queried by the algorithm in any iteration, we include the query result there ($+1$ or $-1$), else the entry is empty which indicates that the pair was not queried by the entire run of the algorithm. This matrix $A$ has the property that for any entry $(i,j)$, if $i$ and $j$ belong to the same cluster and queried then $A(i,j)=+1$ with probability $(1-p)$ and $A(i,j)=-1$ with probability $p$. On the other hand, if $i$ and $j$ belong to different clusters and queried then $A(i,j)=-1$ with probability $(1-p)$ and $A(i,j)=+1$ with probability $p$. Note that the adjacency matrix of $G'$ in any iteration is a submatrix of $A$ which has no empty entry. 

We first look at {\it Phase 1-2C}. At every iteration, our algorithm selects a submatrix of $A$ corresponding to $V' \times V'$ after step 1. This submatrix of $A$ has no empty entry. Let us call it $A'$. We show that if $V'$ contains any subcluster of size $\geq \frac{64k\log{n}}{(2p-1)^4}$, it is retrieved by step 2 with probability at least $1-\frac{1}{n^2}$. In that case, the iteration succeeds. Now the submatrices from one iteration to the other iteration can overlap, so we can only apply union bound to obtain the overall success probability, but that suffices. The probability that in step 2, the algorithm fails to retrieve any cluster of size at least $\frac{64k\log{n}}{(2p-1)^4}$ in any iteration is at most $\frac{1}{n^2}$. The number of iterations is at most $k < n$, since in every iteration except possibly for the last one, $V'$ contains at least one subcluster of that size by a simple pigeonhole principle. This is because in every iteration except possibly for the last one $|V'|=\frac{64k^2\log{n}}{(2p-1)^4}$, and there are at most $k$ clusters. Therefore, the probability that there exists at least one iteration which fails to retrieve the ``large'' clusters is at most $\frac{k}{n^2} \leq \frac{1}{n}$ by union bound. Thus all the iterations will be successful in retrieving the large clusters with probability at least $1-\frac{1}{n}$.

Now, following the same argument as Lemma~\ref{lemma:vertex}, each such cluster will be grown completely by Phase 3-C step (1), and will be output correctly in Phase 3-C step 2.

\begin{lemma}
\label{lemma:efficient1}
Let $c=\frac{64}{(1-2p)^4}$. Whenever $G'$ contains a subcluster of size $ck\log{n}$, it is retrieved by Algorithm 2 in Phase 1-2C with high probability.
\end{lemma}
\begin{proof}
 Consider a particular iteration. Let $N^+(u)$ denote all the neighbors of $u$ in $G'$ connected by $+1$ edges. Let $A'$ denote the corresponding submatrix of $A$ corresponding to $G'$. We have $|V'| \leq N$ ($|V'|=N$ except possibly for the last iteration). Assume, $|V'|=N'$. Also $|V|=n$. 
 
 Let $C_u$ denote the cluster containing $u$. We have
\[ E[|N^+(u)|]=(1-p)|C_u|+p(N'-|C_u|)=pN'+(1-2p)|C_u|\]

% Define $T=pN+\frac{5}{(1-2p)}\sqrt{N\log{n}}$ and $\theta=2p(1-p)N+2\sqrt{N\log{n}}$. For every $u$ and $v$ such that $u \neq v$ and $|N^+(u)|, N^+(v)| \geq T$, if the symmetric difference of these neighborhoods is at most $\theta$ include $u$ and $v$ in the same cluster. This retrieves all clusters of size at least $T$ with probability $1-\frac{1}{n^2}$. 

By the Hoeffding's inequality
\[\Pr(|N^+(u)| \in pN'+(1-2p)|C_u| \pm 2\sqrt{N\log{n}}) \geq 1- \frac{1}{n^4}\]

Therefore for all $u$ such that $|C_u| \geq \frac{8\sqrt{N\log{n}}}{(1-2p)^2}$, we have $|N^+(u)|>pN'+\frac{6}{(1-2p)}\sqrt{N\log{n}}=T(|V'|)$, and for all $u$ such that $|C_u| \leq \frac{4\sqrt{N\log{n}}}{(1-2p)^2}$, we have $|N^+(u)|<pN'+ \frac{6}{(1-2p)}\sqrt{N\log{n}}$ with probability at least $1-\frac{1}{n^3}$ by union bound.

Consider all $u$ such that $|N^+(u)| > T(|V'|)$. Then with probability at least $1-\frac{1}{n^3}$, we have $|C_u| >\frac{4\sqrt{N\log{n}}}{(1-2p)^2}$. Let us call this set $U$. For every $u,v \in U, u \neq v$, the algorithm computes the symmetric difference of $N^+(u)$ and $N^+(v)$ which is
\begin{enumerate}
\item $2p(1-p)N'$ on expectation if $u$ and $v$ belong to the same cluster. And again applying Hoeffding's inequality, it is at most $2p(1-p)N'+2\sqrt{N\log{n}}$ with probability at least $1-\frac{1}{n^4}$.
\item $(p^2+(1-p)^2)(|C_u|+|C_v|)+2p(1-p)(N'-|C_u|-|C_v|)=2p(1-p)N'+(1-2p)^2(|C_u|+|C_v|)$ on expectation if $u$ and $v$ belong to different clusters. Again using the Hoeffding's inequality, it is at least $2p(1-p)N'+(1-2p)^2(|C_u|+|C_v|)-2\sqrt{N\log{n}}$ with probability at least $1-\frac{1}{n^4}$. \end{enumerate}

Therefore, for all $u$ and $v$, either of the above two inequalities fail with probability at most $\frac{1}{n^2}$. 

Now, since for all $u$ if $|N^+(u)| > T(|V'|)$ then $|C_u| >\frac{4\sqrt{N\log{n}}}{(1-2p)^2}$ with probability $1-\frac{1}{n^3}$, we get 

%Hence, with probability at least $1-\frac{2}{n^3}$, we have the symmetric difference to be at least $2p(1-p)N+6\sqrt{N\log{n}}$.
 for every $u$ and $v$ in $U$, if the symmetric difference of $N^+(u)$ and $N^+(v)$ is $\leq 2p(1-p)N'+2\sqrt{N\log{n}}=\theta(|V'|)$, then $u$ and $v$ must belong to the same cluster with probability at least $1-\frac{1}{n^2}-\frac{1}{n^3} \geq 1-\frac{2}{n^2}$. 
 
 Hence, all subclusters of $G'$ that have size at least $\frac{8\sqrt{N\log{n}}}{(1-2p)^2}$ will be retrieved correctly with probability at least $1-\frac{2}{n^2}$. Now since $N'=N=\frac{64k^2\log{n}}{(1-2p)^4}$ for all but possibly the last iteration, we have $\frac{8\sqrt{N\log{n}}}{(1-2p)^2}=\frac{64k\log{n}}{(1-2p)^4}$. Moreover, since there are at most $k$ clusters in $G$ and hence in $G'$, there exists at least one subcluster of size $\frac{64k\log{n}}{(1-2p)^4}$ in $G'$ in every iteration except possibly the last one, which will be retrieved.
 
 Then, there could be at most $k < n$ iterations. The probability that in one iteration, the algorithm will fail to retrieve a large cluster by our analysis is at most $\frac{2}{n^2}$. Hence, by union bound over the iterations, the algorithm will successfully retrieve all clusters in Phase 1-2C with probability at least $1-\frac{2}{n}$.
 \end{proof}

%\paragraph*{Analysis} We are now ready to analyze our algorithm.
%
%We run {\sf CC} on $G'=(V',E')$ such that $|V'|=\Theta(\frac{k^2\log{n}}{(1-2p)^2})$. Since there are at most $k$ clusters in $G'$, there is at least one cluster of size $\Omega(\frac{k\log{n}}{(1-2p)^2})$. Setting the constants right, this crosses the threshold $T$ in the above analysis of {\sf CC}, and hence will be detected correctly with high probability at Phase 1-2C in step (2). 

Now, following the same argument as in Lemma~\ref{lemma:vertex}, each subcluster of size $\frac{64k\log{n}}{(1-2p)^4}$ will be grown completely by Phase 3-C step (1). 

Running time of the algorithm is dominated by the time required to run step 2 of Phase 1-2C. Computing trivially, finding the symmetric differences of $+1$ neighborhoods all ${N}\choose{2}$ pairs requires time $O(N^3)$. We can keep a sorted list of $+1$ neighbors of every vertex is $O(N^2\log{n})$ time. Then, for every pair, it takes $O(N)$ time to find the symmetric difference. This can be reduced to $O(N^\omega)$ using fast matrix multiplication to compute set intersection where $\omega \leq 2.373$. Moreover, since each invocation of this step removes one cluster, there can be at most $k$ calls to it and for every vertex, time required in Phase 3C over all the rounds is $O(k\log{n})$. This gives an overall running time of $O(nk\log{n}+kN^\omega)=O(nk\log{n}+k^{1+2\omega})=O(nk\log{n}+k^{5.746})$. Without fast matrix multiplication, the running time is $O(nk\log{n}+k^7)$.

%By doing a more refined analysis, the number of iterations of Phase $1$ can be reduced to $O(\log{n})$ leading to a running time of $O([nk+k^{4.746}]\log{n})$ using fast matrix multiplication, and $O([nk+k^{6}]\log{n})$ otherwise.

%To see the number of rounds

The query complexity of the algorithm is $O(\frac{nk^2\log{n}}{(2p-1)^4})$ since each vertex is involved in at most $O((\frac{k^2\log{n}}{(2p-1)^4})$ queries within $G'$ and $O(\frac{k\log{n}}{(2p-1)^2})$ queries across the active clusters. Thus we get Theorem~\ref{cor:er-poly}. %\qed

% We can reduce the running time from quasi-polynomial to polynomial, by paying higher in the query-complexity. Suppose, we accept a subgraph extracted from $G'$ as valid and add it to {\sf active} iff its size is $\Omega(k)$. Then note that since $G'$ can contain at most $k^2$ vertices, such a subgraph can be obtained in polynomial time using the algorithm of correlation clustering with noisy input \cite{ms:10}, where all the clusters of size $\Omega(\sqrt{n})$ are recovered on a $n$-vertex graph. Since our ML estimate is correlation clustering, we can employ \cite{ms:10}. For $k \geq \frac{\log n}{\cH^2(p\|1-p)}$, the entire analysis remains valid, and we get a query complexity of $\tilde{O}(nk^2)$ as opposed to $O(\frac{nk}{\cH^2(p\|1-p)})$. If $k < \frac{\log n}{\cH^2(p\|1-p)2}$, then clusters that have size less than $ \frac{\log n}{\cH^2(p\|1-p)}$ are anyway not recoverable. If $k\geq\sqrt{n}$, then the algorithm of \cite{ms:10} recovers all clusters of size at least $O(\sqrt{n})$.
 
%\remove{\begin{theorem*}[restated~\ref{cor:er-poly}]
% There exists a polynomial time algorithm with query complexity $O(\frac{nk^2}{(2p-1)^2})$ for \cc~with error probability $p$, which recovers all clusters of size at least $\Omega(\frac{k\log n}{(2p-1)^2})$. % in $G$.
% \end{theorem*}}
 
 \begin{remark} Readers familiar with the correlation clustering algorithm for noisy input from \cite{bbc:04} would recognize that the idea of looking into symmetric difference of positive neighborhoods is from \cite{bbc:04}. Like \cite{bbc:04}, we need to know the parameter $p$ to design our algorithm. In fact, one can view our algorithm as running the algorithm of \cite{bbc:04} on carefully crafted subgraphs. Developing a parameter free algorithm that works without knowing $p$ remains an exciting future direction.
 \end{remark}
 
 \vspace{0.2in}
 \paragraph*{\bf Unknown $k$}
 Let $c=\frac{64}{(1-2p)^4}$. When the number of clusters $k$ is unknown, it is not possible exactly to determine when the subgraph $G'=(V',E')$ contains $ck^2\log{n}$ sampled vertices. To overcome such difficulty, we propose the following approach of iteratively guessing and updating the estimate of $k$ based on the highest size of $N^{+}(v)$ for $v \in V'$. Let $\ell$ be the guessed value of $k$. We start with $\ell=2$.
 \begin{enumerate}
 \item Guess $k=\ell$
 \item Randomly sample $X$ vertices so that $N=|V'|=c\ell^2\log{n}$ 
 \item For each $v \in V'$, estimate $\hat{C_v}= \frac{1}{(1-2p)}{(|N^{+}(v)|-pN)}$
 \item If $\max_{v}\hat{C_v} > \frac{6\ell\log{n}}{(1-2p)^4}$ then run step 2 of Phase 1-3C on $G'$ with $k=l$, and then move to Phase 3C.
 \item Else set $k=2l$ and move to step (2).
  \end{enumerate}
  
  Clearly, we will never guess $l >2k$, and hence the process converges after at most $\log{k}$ rounds. When $N=c\ell^2\log{n}$, we have $\sqrt{N\log{n}}\leq c\ell\log{n}$ (we must have $\ell^2 \leq n$, otherwise we sample the entire graph). From Lemma~\ref{lemma:efficient1} we get, whenever $\hat{C_v} > \frac{6\ell\log{n}}{(1-2p)^4}$, the actual size of cluster containing $v$ is $\geq \frac{4\ell\log{n}}{(1-2p)^4}$ with high probability. We can then obtain the exact subcluster containing $v$ in $G'$ and grow it fully in Phase 3C with high probability. The query complexity remain the same within a factor of $2$ and running time increases only by a factor of $\log{k}$.
  
  \vspace{0.1in}
 
\paragraph*{\bf Discussion:  Correlation Clustering over Noisy Input.}

In a random noise model, also introduced by \cite{bbc:04} and studied further by \cite{ms:10}, we start with a ground truth clustering, and then each edge label is flipped with probability $p$. \cite{bbc:04} gave an algorithm that recovers all true clusters of size $\geq c_1\sqrt{n\log{n}}$ for some suitable constant $c_1$ under this model. Moreover, if all the clusters have size $\geq c_2\sqrt{n}$, \cite{ms:10} gave a semi-definite programming based algorithm to recover all of them.
Using the algorithm for unknown $k$ verbatim, we can obtain a correlation clustering algorithm for random noise model that recovers all clusters of size $\Omega(\frac{\min(k,\sqrt{n})\log{n}}{(2p-1)^4})$. Since the maximum likelihood estimate of our algorithm is correlation clustering, the true clusters (which is same as the ML clustering) of size $\Omega(\frac{\min(k,\sqrt{n})\log{n}}{(2p-1)^4})$ that the algorithm recovers is the correct correlation clustering output. Therefore, when $k< \frac{\sqrt{n}}{\log{n}}$, we can recover much smaller sized clusters than \cite{bbc:04,ms:10}.

\begin{theorem}
There exists a deterministic polynomial time algorithm for correlation clustering over noisy input that recovers all the underlying true clusters of size  at least $c_3\min{(k,  \sqrt{n})}\log{n}$ for a suitable constant $c_3$ with high probability.
\end{theorem}

\section{Non-adaptive Algorithm and the Stochastic Block Model} \label{sec:na}
In this section, we consider the case when all queries must be made upfront that is adaptive querying is not allowed. We show how our adaptive algorithms can be modified to handle such setting. Specifically, for $k=2$, we show nonadaptive algorithms are as powerful as adaptive algorithms, but for $k \geq 3$, unless the maximum to minimum cluster size is bounded, there is a significant advantage gained by using adaptive algorithm. 

First, let us note that when there are only two clusters, and the oracle  gives correct answers, then it is possible
to recover the clusters with only $n-1$ queries. Indeed, just query every element with a fixed element. It is also easy to see than 
$\Omega(n)$ queries are required (since our lower bound of Theorem~\ref{thm:faulty} is valid in this special case).

On the other hand, consider the case when there are $k >2$ clusters, and the oracle is perfect. We show that any deterministic algorithm would require $\Omega(n^2)$ queries. This is in stark contrast with our adaptive algorithms which are all deterministic and achieve significantly less query complexity.

\begin{claim}
Assume there are $k \ge 3$ clusters and the minimum size of a cluster is $r$. Then any deterministic nonadaptive algorithm must make 
$\Omega(\frac{n^2}{r})$ queries, even when the when query answers are perfect. This shows that adaptive algorithms are much more powerful than their nonadaptive counterparts.
\end{claim}
\begin{proof}
Consider a graph with $n$ vertices and there will be an edge between two vertices if the deterministic nonadaptive algorithm makes queries between them. Assume the number of queries made is at most $\frac{n^2}{4r}$. Then, using Tur\'an's theorem, this graph 
must have an independent set of size at least $\frac{n}{n/2r +1} \approx 2r$. We can create an closeting instance with three clusters:
one large cluster with $n-2r$ vertices, and two small clusters with size $r$ each, where the union of the later two constitutes the
independent set. Since the algorithm makes no query within the later two cluster, there will be no way to identify them. Hence the number of queries for any nonadaptive deterministic algorithm must be more than $\frac{n^2}{4r}$.
\end{proof}

Moving on to the faulty oracle case, we prove the following theorem.

\begin{theorem}\label{theorem:nonadaptive}
 For number of clusters $k=2$, there exists an $O(n\log{n})$ time nonadaptive algorithm that recovers the clusters with high probability with query complexity $O(\frac{n\log{n}}{(1-2p)^4})$.

 For $k \geq 3$, if $R$ is the ratio between maximum to minimum cluster size, then there exists a randomized nonadaptive algorithm that recovers all clusters with high probability with query complexity $O(\frac{Rnk\log{n}}{(1-2p)^2})$. Moreover, there exists a computationally efficient algorithm for the same with query complexity $O(\frac{Rnk^2\log{n}}{(1-2p)^4})$.
%$\bullet$ For $k \geq 3$, if the minimum cluster size is $r$, then any deterministic non-adaptive algorithm must make $\Omega(\frac{n^2}{r})$ queries even when query answers are perfect to recover the clusters exactly. This shows that adaptive algorithms are much more powerful than their nonadaptive counterparts.
 \end{theorem}

\paragraph*{Non-adaptive with $k=2$:}
For $k=2$, the algorithm is as follows. It constructs the graph $G'=(V',E')$ by randomly sampling $N=4c\log{n}$ vertices where $c=\Theta(\frac{1}{(1-2p)^4})$ and querying all ${|V'|}\choose{2}$ pairs as well as all $(u,v)$ where $u \in V \setminus V'$ and $v \in V'$. Note that this is quite different from random querying.

$G'$ then contains at least one subcluster of size at least $2c\log{n}=\frac{N}{2}$, which is recovered by running the computationally efficient algorithm from Section~\ref{sec:efficient}. Using the query answers of $(u,v)$ where $u \in V \setminus V'$ and $v \in V'$, the subcluster is then grown fully. Finally, all the other vertices are put in a separate cluster.

The algorithm running time is $O(n\log{n})$ from the running time discussion of our computationally efficient adaptive algorithm for known $k$. This improves upon \cite{DBLP:journals/corr/MitzenmacherT16,chen2012clustering,chen2014clustering}.

\paragraph*{Non-adaptive with $k\geq 3$:}
\remove{When the number of clusters is $k \geq 3$, as see in Section~\ref{sec:lower}, any non-adaptive algorithm must make $\Omega(n^2)$ queries. In that case, we will be trivially be running {\sf CC} on $G$ after querying all the edges.

However, if the clusters are all nearly balanced, then we can again modify our adaptive algorithm in the same vein as $k=2$ and recover all clusters of size $c_3\min{(k,  \sqrt{n})}\log{n}$ for a suitable constant $c_3$.}

Let $R \geq 1$ be the ratio of the maximum to minimum cluster size. When the minimum size cluster is small, in Appendix~\ref{sec:nq}, we provide a lower bound of $\Omega(n^2)$ for any deterministic algorithm.
Our algorithm simply creates $G'$ by randomly and uniformly sampling $\Theta(\frac{Rk^2\log{n}}{(1-2p)^4})$ vertices from $G$. It then queries all $(u,v) \in V' \times V'$. We here assume $\Theta(\frac{Rk^2\log{n}}{(1-2p)^4}) < n$, otherwise $G'$ is the entire fully-queried graph $G$. The query complexity is therefore, $O(\frac{Rnk^2\log{n}}{(1-2p)^4})$.

Since, we sample the vertices uniformly at random, the minimum number of vertices selected from any cluster with high probability using the Chernoff bound is $O(\frac{Rnk\log{n}}{(1-2p)^4})$. Now, again following the algorithm of Section~\ref{sec:efficient}, we can recover all these subclusters exactly with high probability--the remaining queries are then used to grow them fully. The running time of the algorithm is same as the running time of its adaptive version.

To obtain an information theoretic optimal result within an $O(\log{n})$ factor, instead of sampling $\Theta(\frac{Rk^2\log{n}}{(1-2p)^4})$ vertices, we sample $\Theta(\frac{Rk\log{n}}{(1-2p)^2})$ vertices from $G$ to construct $G'$ and then issue all pairwise queries $(u,v) \in V \times V'$. Then, by the same argument, the minimum size of any subcluster in $G'$ is at least $\Theta(\frac{\log{n}}{(1-2p)^2})$ with high probability which can be recovered by using the algorithm for detecting heaviest weight subgraph from Section~\ref{sec:info-theory}.

\subsection{The Stochastic Block Model}\label{sec:nq}

Our model of faulty oracle is closely related to the stochastic block model. Indeed, if all $\binom{n}2$ queries are performed with the faulty oracle $\cO_{p,q}$, we exactly recover the adjacency matrix of usual stochastic block model. When we are performing a fixed number $Q < \binom{n}2$ of queries to the  oracle, we can think of that as a generalization of the stochastic block model, where only $Q$ entrees of the adjacency matrix of the stochastic block model is being provided to us.
Once crucial point about our model is that though, we can adaptively query to carefully select the entries of the adjacency matrix of the stochastic block model to ensure recovery of the clustering.

Let us, consider the case when all of the $Q$ queries are made nonadaptively. This is still a generalization of stochastic block model (in which case $Q = \binom{n}2$). Assume the prior probability of each element being assigned to any cluster is uniform.  Since each query involves two elements, this means that the average number of queries an element is involved in is $\frac{2Q}{n}$. Using Markov inequality, we can say that there exists at least $\frac{n}2$ elements $U$, each of which are involved in at most $\frac{4Q}{n}$ queries.

Now we can restrict ourselves to finding the clustering among only such $\frac{n}{2}$ elements each of which are involved in at most $\frac{4Q}{n}$ queries. Now let us just take any two clusters $V_1$ and $V_2$ and a fixed element $v\in V_1 \cap U$. We obtain $K = \frac{n}{2k}$ different equiprobable clusterings by interchanging $v$ with the elements of $V_2 \cap U$. Let us consider the task of distinguishing between these $K$ hypotheses, by looking the query answers.

Now, we can use a generalized Fano's inequality from  \cite{polyanskiy2010arimoto}[Thm.~4], where we consider Renyi divergence of order $\frac12$, to have, 
\begin{align*}
 -2\log&\Big(\sqrt{\frac{1-P_e}K} + \sqrt{P_e(1-\frac1K)}\Big)  \\
& \le - \log \sum_y (\frac{1}{K} \sum_{j=1}^K \sqrt{Q_j(y)})^2
\end{align*}
where $P_e$ the probability of error of this hypothesis testing problem. This implies,
\begin{align*}
 &\Big(\sqrt{\frac{1-P_e}K} + \sqrt{P_e(1-\frac1K)}\Big)^2  \ge  1 - \cH^2(Q_i \| Q_j)\\
& \ge 1 -\Big(1-(1-\cH^2(p\|q))^{\frac{8Q}{nk}}\Big)
   = (1-\cH^2(p\|q))^{\frac{8Q}{nk}},
 \end{align*}
 where we have used the fact that each element considered  can influence at most $\frac{4Q}{nk}$ query answers on average by this interchange.
 Again, if we assume $p \sim \text{Bernoulli}\Big(\frac{a \log n}{n}\Big)$ and $q \sim \text{Bernoulli}\Big(\frac{b \log n}{n}\Big)$, a particular regime of interest for stochastic block model, 
 then, 
 \begin{align*} 
 &\sqrt{\frac{k}{n}} +\sqrt{P_e} \\
 & \ge \Big(\frac{\sqrt{ab}\log n}{n} 
 + \sqrt{(1-\frac{a\log n}{n})(1-\frac{b \log n}{n})}\Big)^{\frac{4Q}{nk}} \\
 &= n^{-\Big(\frac{a+b}{2}-\sqrt{ab}- \frac{ab \log n}{n}\Big)\frac{4Q}{n^2k}}.
\end{align*}

This implies,
$
\sqrt{P_e} \ge n^{-\Big(\frac{a+b}{2}-\sqrt{ab}\Big)\frac{4Q}{n^2k}} - \sqrt{k} n^{-1/2}.
$
In particular, if 
$
\Big(\frac{a+b}{2}-\sqrt{ab}\Big)\frac{4Q}{n^2k} <\frac12,
$
then $P_e > 0$. %\frac{1}{n}$. 
Hence,  $P_e > \frac{1}{n}$ if 
$
\sqrt{a} - \sqrt{b} <\frac{n}{2}\sqrt{\frac{k}Q}.
$

Note that when $Q = \binom{n}2$, the maximum possible value, we get  $\sqrt{a}-\sqrt{b} < \sqrt{\frac{k}{2}}\implies P_e > 0,$--this is slightly suboptimal by a factor of $\sqrt{2}$ than what is known for the stochastic block model \cite{DBLP:conf/focs/AbbeS15, mossel2015consistency}. Tightening the constant, and getting matching upper bound for arbitrary $Q$ are interesting future work. However, note that, our tools are not specialized for this regime of stochastic block models, and the result works for general values of $Q$, not only the corner point of $Q= \binom{n}2$.

Now to extend this argument, to the case where adaptive querying is allowed, is difficult. % - since it is hard to form the equiprobable clusterings. 
Therefore we have to rely on the  general technique of Theorem \ref{thm:faulty}.

\begin{remark}\label{rem:na}
There is another  different version of Fano's inequality that we can use here - form  \cite{han1994generalizing}[Thm.~7], that says
the probability of error of this hypothesis testing problem is:
$$
P_e \ge 1- \frac{\frac{4Q}{nk}(D(p \| q)+D(q\|p)) +\ln 2}{\log\frac{n}{2k}}.
$$
This says that the number of nonadaptive queries must be at least $\Omega(\frac{nk \log n}{D(p\|q)+D(q\|p)})$ to recover the clustering with positive probability (this is indeed a lower bound for balanced clustering). As we have seen from Theorem~\ref{theorem:nonadaptive},
 this bound is tight. 
\end{remark}

%% file: appendix-algo.tex
\section{Algorithms}\label{appen:faultyub}
\subsection{Proofs of the claims in Lemma~\ref{lemma:mlG'}}
\begin{proof}[Proof of Claim~\ref{claim:0}]
For an $i: |V'_i| \ge c' \log n,$ we have
\begin{align*}
\avg \sum_{s, t \in V'_i, s<t}\omega_{s,t} = \binom{|V'_i|}{2}((1-p)-p) = (1-2p)\binom{|V'_i|}{2}.
\end{align*}
Since $\omega_{s,t}$ are independent binary random variables, using the Hoeffding's inequality (Lemma \ref{lem:hoef1}),
\begin{align*}
\Pr\Big( \sum_{s, t \in V'_i, s<t}\omega_{s,t} \le \avg \sum_{s, t \in V'_i, s<t}\omega_{s,t} - u \Big) \le e^{-\frac{ u^2 }{2\binom{|V'_i|}{2}}}.
\end{align*}
Hence,
\begin{align*}
\Pr\Big( \sum_{s, t \in V'_i, s<t}\omega_{s,t} >(1-\delta) \avg \sum_{s, t \in V'_i, s<t}\omega_{s,t} \Big) \ge 1 - e^{-\frac{ \delta^2(1-2p)^2 \binom{|V'_i|}{2} }{2}}.
\end{align*}
Therefore with high probability (here the success probability is even $> 1-\frac{1}{n^{\log{n}}}$)
\begin{align*}
& \sum_{s, t \in V'_i, s<t}\omega_{s,t} > (1-\delta) (1-2p)\binom{|V'_i|}{2} \\
& \ge  (1-\delta) (1-2p)\binom{c' \log n}2 >
\frac{c'^2}{3}(1-2p) \log^2 n,
\end{align*} for an appropriately chosen $\delta$ (say $\delta=\frac{1}{4}$).

So, when processing $G'$, the algorithm must return a set $S$ such that $|S| \ge c' \sqrt{\frac{2(1-2p)}{3}} \log n=c''\log{n}$ (define $c''=c' \sqrt{\frac{2(1-2p)}{3}}$) with probability $> 1-\frac{1}{n^{\log{n}}}$ -  since otherwise $$\sum_{i, j \in S, i < j}\omega_{i,j} < \binom{c' \sqrt{\frac{2(1-2p)}{3}} \log n}{2} < \frac{c'^2}{3}(1-2p) \log^2 n.$$

Now let $S \nsubseteq V_i$ for any $i$. Then $S$ must have intersection with at least $2$ clusters. Let $V_i \cap S = C_i$ and
let $j^\ast = \arg \min_{i: C_i \neq \emptyset} |C_i|$. We claim that,
\begin{equation}\label{eq:reduc}
\sum_{i, j \in S, i < j}\omega_{i,j}  < \sum_{i, j \in S \setminus C_{j^\ast}, i < j}\omega_{i,j},
\end{equation}
with high probability.
Condition \eqref{eq:reduc} is equivalent to,
\[
\sum_{i, j \in  C_{j^\ast}, i < j}\omega_{i,j} + \sum_{i \in C_{j^\ast}, j \in S \setminus C_{j^\ast}} \omega_{i,j} <0. \tag{I}
\]
However this is true because,
\begin{enumerate}
\item $
\avg \Big(\sum_{i, j \in  C_{j^\ast}, i < j}\omega_{i,j} \Big) = (1-2p) \binom{|C_{j^\ast}|}{2}$ and 
$\avg\Big(\sum_{i \in C_{j^\ast}, j \in S \setminus C_{j^\ast}} \omega_{i,j} \Big) = - (1-2p)|C_{j^\ast}|\cdot |S\setminus C_{j^\ast}|.
$ Note that $|S\setminus C_{j^\ast}| \geq |C_{j^\ast}|$. Hence the expected value of the L.H.S. of (I) is negative.

\item As long as $|C_{j^\ast}| \ge \frac{12\sqrt{\log{n}}}{(1-2p)}$ we have, from Hoeffding's inequality,
\begin{align*}
&\Pr\Big(\sum_{i, j \in  C_{j^\ast}, i < j}\omega_{i,j}  \ge (1+\nu) (1-2p) \binom{|C_{j^\ast}|}{2}\Big) \\
&\le e^{-\frac{\nu^2(1-2p)^2\binom{|C_{j^\ast}|}{2}}{2}} = n^{-36\nu^2}.
\end{align*}
While at the same time, 
\begin{align*}
& \Pr\Big( \sum_{i \in C_{j^\ast}, j \in S \setminus C_{j^\ast}} \omega_{i,j}  \ge - (1-\nu) (1-2p)|C_{j^\ast}|\cdot |S\setminus C_{j^\ast}|\Big) \\
& \le e^{-\frac{\nu^2 (1-2p)^2 |C_{j^\ast}|\cdot |S\setminus C_{j^\ast}|}{2}} = n^{-72\nu^2}.
\end{align*}
Setting $\nu=\frac{1}{4}$ (say), of course with high probability (probability at least $1-\frac{2}{n^{2.25}}$)
 $$
\sum_{i, j \in  C_{j^\ast}, i < j}\omega_{i,j} + \sum_{i \in C_{j^\ast}, j \in S \setminus C_{j^\ast}} \omega_{i,j} <0.
$$
\item When  $|C_{j^\ast}| < \frac{12\sqrt{\log n}}{(1-2p)}$, let $|C_{j^\ast}|=x$. We have,
$$
\sum_{i, j \in  C_{j^\ast}, i < j}\omega_{i,j} \le \binom{|C_{j^\ast}|}{2} \le \frac{x^2}{2}.
$$
While at the same time, 
\begin{align*}
& \Pr\Big( \sum_{i \in C_{j^\ast}, j \in S \setminus C_{j^\ast}} \omega_{i,j}  \ge -(1-\nu) (1-2p)|C_{j^\ast}|\cdot |S\setminus C_{j^\ast}|\Big) \\
&\le e^{-\frac{\nu^2 (1-2p)^2 |C_{j^\ast}|\cdot |S\setminus C_{j^\ast}|}{2}} \leq e^{-\frac{\nu^2 (1-2p)^2 x (|S|-x)}{2}}
\end{align*}
If $x \geq \sqrt{\frac{3}{2(1-2p)}}$, then $x (|S|-x)\geq \frac{2x|S|}{3}=\frac{2c'\log{n}}{3} \geq \frac{64\log{n}}{(1-2p)^2}$, where the second inequality followed since $x < \frac{S}{3}$.
Hence, in this case, again setting $\nu=\frac{1}{4}$ and noting the value of $S$ and the fact $|C_{j^\ast}| < \frac{12\sqrt{\log n}}{(1-2p)}$, with probability at least $1-\frac{1}{n^2}$,
$$
\sum_{i, j \in  C_{j^\ast}, i < j}\omega_{i,j} + \sum_{i \in C_{j^\ast}, j \in S \setminus C_{j^\ast}} \omega_{i,j} <0.
$$

If $x < \sqrt{\frac{3}{2(1-2p)}}$, then $(S-x) > \frac{48x\log{n}}{(1-2p)}$. Hence $E[\sum_{i \in C_{j^\ast}, j \in S \setminus C_{j^\ast}} \omega_{i,j}] \leq -(1-2p)x(S-x)<- 48\log{n}\frac{x^2}{2}$.

Hence by Hoeffding's inequality,
\[\Pr\Big( \sum_{i \in C_{j^\ast}, j \in S \setminus C_{j^\ast}} \omega_{i,j}  \ge -\frac{x^2}{2}\Big) \leq e^{-\frac{2*47*47x^4\log^2{n}}{|C_{j^\ast}||S \setminus C_{j^\ast}|}}\leq e^{-\frac{2*47*47x^3\log^2{n}}{|S|}} <<\frac{1}{n^2}\]

\end{enumerate}
Hence \eqref{eq:reduc} is true with probability at least $1-\frac{4}{n^2}$. But then the algorithm would not return $S$, but will return $S \setminus C_{j^\ast}$. Hence, we have run into a contradiction. This means $S \subseteq V_i$ for some $V_i$. 
\end{proof}

\vspace{0.2in}
\begin{proof}[Proof of Claim~\ref{claim:1}]
From Claim~\ref{claim:0} with probability at least $1-\frac{4}{n^2}$, $S \subseteq V_i$ and $$\sum_{i, j \in S, i < j}\omega_{i,j} \geq \frac{c'^2}{3}(1-2p) \log^2 n.$$
Suppose if possible $|S|=x < c\log{n}=\frac{c'\log{n}}{6}$. Then
$$E[\sum_{i, j \in S, i < j}\omega_{i,j} ]<\frac{x^2}{2}(1-2p)$$

Hence, by the Hoeffding's inequality
\begin{align*}
\Pr\Big(\sum_{i, j \in S, i < j}\omega_{i,j} \geq \frac{c'^2}{3}(1-2p) \log^2 n\Big) 
&\leq e^{-\frac{(1-2p)^2\Big(\frac{c'^2}{3}\log^2{n}-\frac{x^2}{2}\Big)^2}{x^2}}\\
&\leq e^{-\frac{(1-2p)^2\Big(\frac{c'^2}{4}\log^2{n}\Big)^2}{x^2}}<<\frac{1}{n^2}
\end{align*}
Therefore, $|S| \geq c\log{n}$ with probability at least $1-\frac{5}{n^2}$.

\end{proof}

\vspace{0.2in}
\begin{proof}[Proof of Claim~\ref{claim:2}]
If the algorithm returns a set $S$ with $|S| > c\log{n}$ then $S$ must have intersection with at least $2$ clusters in $V$. Now following the same argument as in Claim \ref{claim:0} to establish Eq. \eqref{eq:reduc}, we arrive to a contradiction, and $S$ cannot be returned.
\end{proof}